\documentclass[twoside]{article}

\newif\ifpreprintflag
\preprintflagtrue

\ifpreprintflag
\usepackage[preprint]{aistats2027}
\else
\usepackage{aistats2027}
\fi 

\usepackage[numbers,square]{natbib}

\usepackage[utf8]{inputenc}
\usepackage[T1]{fontenc}
\usepackage{amsmath}
\usepackage{amsfonts}
\usepackage{amssymb}
\usepackage{amsthm}
\usepackage{nicefrac}
\usepackage{bm}
\usepackage{booktabs}
\usepackage{multirow}
\usepackage{graphicx}
\usepackage{subcaption}
\usepackage{wrapfig}
\usepackage{placeins}
\usepackage{circledsteps}
\usepackage{comment}
\usepackage{microtype}
\usepackage{xcolor}
\usepackage{algorithm}
\usepackage{algpseudocode}
\usepackage{enumitem}  
\usepackage{soul}      
\usepackage{url}
\usepackage{hyperref}
\usepackage[capitalize]{cleveref}

\theoremstyle{plain}
\newtheorem{theorem}{Theorem}
\newtheorem{proposition}{Proposition}

\theoremstyle{definition}
\newtheorem{definition}{Definition}

\crefname{definition}{Def.}{Defs.}
\Crefname{definition}{Def.}{Defs.}
\crefname{proposition}{Prop.}{Props.}
\Crefname{proposition}{Prop.}{Props.}
\crefname{theorem}{Thm.}{Thms.}
\Crefname{theorem}{Thm.}{Thms.}
\crefname{corollary}{Cor.}{Cors.}
\Crefname{corollary}{Cor.}{Cors.}
\crefname{lemma}{Lem.}{Lems.}
\Crefname{lemma}{Lem.}{Lems.}
\crefname{remark}{Rem.}{Rems.}
\Crefname{remark}{Rem.}{Rems.}
\crefname{assumption}{Asm.}{Asms.}
\Crefname{assumption}{Asm.}{Asms.}
\crefname{section}{Sec.}{Secs.}
\Crefname{section}{Sec.}{Secs.}
\crefname{appendix}{Apx.}{Apxs.}
\Crefname{appendix}{Apx.}{Apxs.}
\crefname{subsection}{Sec.}{Secs.}
\Crefname{subsection}{Sec.}{Secs.}
\crefname{subappendix}{Apx.}{Apxs.}
\Crefname{subappendix}{Apx.}{Apxs.}
\crefname{subsubappendix}{Apx.}{Apxs.}
\Crefname{subsubappendix}{Apx.}{Apxs.}
\crefname{figure}{Fig.}{Figs.}
\Crefname{figure}{Fig.}{Figs.}
\crefname{table}{Tab.}{Tabs.}
\Crefname{table}{Tab.}{Tabs.}
\crefname{equation}{Eq.}{Eqs.}
\Crefname{equation}{Eq.}{Eqs.}
\creflabelformat{equation}{#2#1#3}

\newcommand{\stackcell}[1]{\begin{tabular}[c]{@{}l@{}}#1\end{tabular}}

\begin{document}

\runningtitle{A Theory-Grounded Approach to Personalized Suicide-Risk Forecasting}
\runningauthor{Yacoby et al.}

\addtocontents{toc}{\protect\setcounter{tocdepth}{-1}}  

\twocolumn[
\aistatstitle{Latent Similarity Gaussian Processes: A Theory-Grounded Approach to Personalized Suicide-Risk Forecasting for Clinical Decision-Support}
\aistatsauthor{
  Yaniv Yacoby \And
  Weiwei Pan \And
  Hope Neveux \And
  Taylor C.~McGuire
}
\aistatsaddress{
  Wellesley College \\ \texttt{yy109@wellesley.edu} \And
  Harvard University \\ \texttt{weiweipan@g.harvard.edu\phantom{AAA}} \And
  Harvard University \\ \texttt{hopeneveux@fas.harvard.edu} \And
  Harvard University \\ \texttt{\phantom{AAA}tmcguire@g.harvard.edu}
}
\ifpreprintflag
\aistatsauthor{
  Franchesca Castro-Ramirez \And
  Anushka R.~Patel \And
  Matthew K.~Nock
}
\aistatsaddress{
  Harvard University \\ \texttt{framirez@g.harvard.edu} \And
  Harvard Medical School \\ Massachusetts General Hospital \\ \texttt{apatel132@mgh.harvard.edu} \And
  Harvard University \\ \texttt{nock@wjh.harvard.edu}
}
\fi 
]

\begin{abstract}
Forecasting suicide risk is difficult due to the high heterogeneity of patients and the low base rate of suicide-related events (SREs). We present Latent Similarity Gaussian Processes (LSGPs), which embed patients in a continuous latent space to jointly model similarity and forecast risk. By selectively drawing information from latent peers, LSGPs better capture individualized risk trajectories, generalizing nomothetic (pooled), idiographic (per-patient), and hierarchical frameworks. Our contributions are: (1) an identifiable two-channel Similarity Kernel; (2) proof that the standard model-fitting algorithm, mean-field variational inference, collapses LSGPs to nomothetic models, along with a fix; and (3) empirical results on intensive longitudinal suicide data showing LSGPs outperform nomothetic, idiographic, and hierarchical models for next-week risk forecasting, with the largest gains in forecasting first-occurrence SREs. 
\end{abstract}

\section{Introduction} \label{sec:intro}

Machine Learning (ML) systems for suicide-risk forecasting face dual demands: providing accurate forecasts and actionable insights to clinicians. 
Existing approaches largely fall between two extremes: \emph{nomothetic} (population-level/pooled) and \emph{idiographic} (per-patient/individualized) models\ifpreprintflag\footnote{We adopt the standard clinical psychology terms ``nomothetic'' and ``idiographic'' to ground our framework in clinical practice and unify overlapping ML terminology.}\fi \citep[e.g.,][]{wang2023idiographic,pigoni2024machine}.
Hierarchical (mixed-effects) models bridge these extremes by decomposing forecasts into a shared population trend (fixed effects) and patient-specific deviations (random effects).
All three model classes make implicit assumptions about the data that limit their clinical utility.

\textbf{Nomothetic models assume that factors relevant to suicide risk are either observed or safely ignorable.}
By sharing a single global trend across all patients, nomothetic models treat any unobserved, patient-specific variation as uncorrelated noise. 
Their use therefore implicitly assumes that every meaningful driver of suicide risk has been observed; however, in clinical reality, pathways to suicide are heterogeneous~\citep{van2000pathways,coppersmith2024heterogeneity,kaurin2022integrating}, shaped by patient-specific contexts that may be absent from, incompletely recorded in, or difficult to extract from standard study instruments and health records~\citep{madden2016missing}. 
For example, important determinants of near-term risk include acute stressors~\citep{gradus2010acute}, trauma-related responses~\citep{krysinska2010post,kaminsky2015epigenetic}, interpersonal dynamics~\citep{kalesan2018role,orlins2021characteristics}, and changes in social circumstances~\citep{kposowa2001unemployment,madigan2023socioeconomic,stack2021contributing}. 
When an individual's baseline risk depends on unobserved factors, a single global trend may fail to alert clinicians in time to act.

\textbf{Idiographic models assume no commonalities in pathways to suicide.}
Idiographic approaches model each patient independently; in doing so, they discard shared, population-level dynamics.
For example, sleep disturbance~\citep{bernert2017objectively,harris2020sleep}, hopelessness~\citep{weishaar1992hopelessness,riera2024clinical}, and prior suicide attempts~\citep{chan2014predictors,ribeiro2016self} are all known to increase suicide risk. And because suicide attempts occur with a low base rate, idiographic models are prone to overfitting~\citep{hang2025improving}, especially during early monitoring when data is sparse and when insights are most needed.

\textbf{Hierarchical models assume patients represent deviations from a single norm.}
While hierarchical models bridge nomothetic and idiographic extremes, they assume patient trends are random deviations from a single global trend. In reality, patients inhabit a continuous spectrum of dynamic subtypes, governed by distinct, potentially conflicting trends \citep{hang2025improving}. Under sparse event histories---typical in suicide-risk monitoring---this assumption forces risk estimates for patients with no prior attempts to default to the global population trend, which may represent them poorly. This failure is clinically critical, as a substantial proportion of suicide fatalities occur on the first attempt~\citep{jordan2020characteristics}.

\textbf{Modeling patient similarity addresses these shortcomings and may support clinical decision-making and treatment planning.}
Explicitly modeling latent patient similarity overcomes key limitations of existing approaches by enabling cross-patient information sharing: when a patient's data is limited, the model leverages patterns from similar patients to forecast their future risk.
Beyond predictive gains, similarity also offers a foundation for clinician-AI collaboration by supporting case-based reasoning~\citep{sharafoddini2017patient}.
Prior work demonstrates substantial heterogeneity in how at-risk patients respond to treatments, even in major interventions like psychiatric hospitalization~\citep{ross2024estimated}.
Consequently, two patients with the same risk score may benefit from entirely different interventions---such as safety planning~\citep{stanley2018comparison}, lethal-means counseling~\citep{miller2020changes}, or family engagement~\citep{asarnow2017cognitive}. 
By learning patient similarity, we can expose historical precedents that allow clinicians to evaluate which interventions succeeded for similar patients and why, assess whether off-survey context renders these precedents relevant, and adapt care accordingly, narrowing a vast, heterogeneous space of clinical actions down to a few targeted choices.

\textbf{Contributions.} Translating this clinical vision into practice requires a model that meets three demands: (1) \emph{predictive power} via accurate individualized forecasts, (2) a \emph{mathematically well-posed mechanism} for learning patient similarity, and (3) \emph{scalable inference} whose objective does not itself bias the model toward purely nomothetic/idiographic regimes.
We meet these demands as follows:

\textbf{(A) Model Class: Latent Similarity Gaussian Processes (LSGPs).}
LSGPs generalize nomothetic, idiographic, and hierarchical models (\cref{apx:hlm-function-space}) by allowing each patient's forecast to be informed by others in direct proportion to their learned dynamic similarity. Formally, LSGPs are built as Latent Variable Multi-Output Gaussian Processes (LVMOGPs~\citep{dai2017efficient}), equipped with a family of Similarity Kernels with two channels: trend and relative-amplitude similarities. We prove these channels are identifiable: distinct parameterizations of a Similarity Kernel producing the same overall kernel function decompose into \emph{identical} channels, ensuring that clinical interpretations are invariant to model parameterization.

\textbf{(B) Theoretical Analysis of Personalization Collapse in the Standard Model-Fitting Algorithm---Mean-Field Variational Inference (MFVI)---and a Scalable Fix.}
We prove that the standard method for fitting LSGPs at scale (MFVI) can induce \emph{personalization collapse}: even when the model is sufficiently flexible to capture patient heterogeneity, MFVI penalizes non-nomothetic models. We characterize the gap between the mean-field and structured bounds as a Bregman information that vanishes only at a fully nomothetic model.
To avoid this collapse without sacrificing scalability, we adapt the P\'olya-Gamma-augmented scheme~\citep{polson2013bayesian,wenzel2019efficient} that restores conditional Gaussian conjugacy between local and global variables, preserving the coupling MFVI discards while keeping coordinate-ascent updates tractable.

\textbf{(C) Application to Data from an Intensive Longitudinal Suicide-Risk Study.}
We evaluate the LSGP on intensive Ecological Momentary Assessment (EMA) data tracking daily affect and suicidal thoughts with the goal of predicting next-week suicide-related events (SREs)---self-injurious behavior performed with non-zero intent to die or a psychiatric hospital presentation intended to avert a suicide attempt~\citep{nock2026using}. 
LSGPs outperform nomothetic, idiographic, and hierarchical baselines in within-patient forecasting, measured by per-patient average precision.
Their largest gains occur on patients with no SREs in the training set who experience an SRE in the test set. Where idiographic models fail due to sparse positive labels, nomothetic models blur individual dynamics into a cohort average, and hierarchical models collapse directly back to that single global norm, the LSGP selectively informs forecasts by drawing on patients with similar risk trajectories.
Further, we find that the learned similarity channels are not explained by static demographics alone, suggesting that the model isolates novel temporal patterns---a starting point for future clinical validation.

\textbf{Impact Relative to Related Work.}
This work contributes to three areas (expanded in \cref{apx:related-work}):
\emph{(i)~Machine Learning:} Scalable GPLVM-family models \citep[e.g.,][]{titsias2010bayesian,martens2019decomposing,dai2017efficient} typically rely on MFVI \citep[e.g.,][]{hensman2013gaussian,lalchand2022generalised}. While MFVI pathologies are known in other models \citep[e.g.,][]{turner2011two,foong2020expressiveness,yacoby2020failure,coker2022wide,yacoby2022mitigating}, we uncover and formalize a new failure mode where MFVI penalizes LSGPs toward the nomothetic extreme.
\emph{(ii)~Suicide Prevention:} EMA studies primarily predict suicidal thoughts rather than SREs \citep[e.g.,][]{kleiman2017examination,czyz2023ecological,lei2023ecological,wang2023idiographic}, which are rarer and harder to forecast. We show that the LSGP improves per-patient SRE forecasts---particularly for first-time SREs---by ensuring the learned patient similarity reflects genuine data-driven trends rather than inference artifacts.
\emph{(iii)~Implementation Science:} Unlike existing clinical retrieval tools that compare patients based on static feature snapshots or latent embeddings \citep[e.g.,][]{sharafoddini2017patient,oei2021using}, LSGPs learn similarity directly in \emph{function space}---matching patients by \emph{how} their survey responses map to suicide risk. By surfacing historical precedents that share true underlying dynamics rather than feature overlap, LSGPs provide a foundation for future decision-support systems that enable clinicians to synthesize their own contextual knowledge with the model's forecasts via case-based reasoning.

\section{MODEL CLASS} \label{sec:method}

We present the LSGP, an LVMOGP equipped with a \emph{Similarity Kernel} (\cref{def:iskernel}). 
By design, the LSGP quantifies similarity between patients' predictive \emph{functions} through its kernel, which ensures similarity is well-posed and uninfluenced by the latent variance scale. 

\textbf{Notation.}
Let $N$ and $I$ be the number of patients and surveys, respectively. 
Let $x_i \in \mathbb{R}^{D_x}$ be the $i$-th survey for patient $n_i$ at time $t_i$, comprising $0$--$10$ Likert items (affect, suicidal intent/urges, etc.) and survey metadata (e.g., response latency; see \cref{apx:setup-u01}). We predict a binary label $y_i \in \{0, 1\}$ indicating whether patient $n_i$ experienced an SRE within 7 days post-$t_i$. The full dataset is $\mathcal{D} = (X, Y)$, with patient-specific subsets $\mathcal{D}_n = (X_n, Y_n) = (\{x_i \mid n_i = n\}, \{y_i \mid n_i = n\})$ varying in size due to differing response rates.

\textbf{Goal.}
Given $\mathcal{D}$, as well as a \emph{new} survey $x_n^*$, we want to predict whether patient $n$ will experience an SRE, $y_n^*$, during the week that follows.

\textbf{Model.}
Patient heterogeneity naturally motivates the following model:
\begin{align}
    &\Circled{A}\ \ z_n \sim p(z) = \mathcal{N}(0, \mathbb{I}_{D_z}), \\
    &\Circled{B}\ \ \hat{x}_i(z_{n_i}) = [\, x_i^\intercal \;\; z_{n_i}^\intercal \,]^\intercal \in \mathbb{R}^{D_x + D_z}, \\
    &\Circled{C}\ \ F \mid \widehat{X}(Z); \theta \sim \mathcal{N}( m_\theta(X), K_\theta(\widehat{X}(Z), \widehat{X}(Z)) ), \\
    &\Circled{D}\ \ y_i \mid f_i \sim \mathrm{Bernoulli}( \mathrm{sigmoid}(f_i) ).
\end{align}
Here, $\hat{x}_i(z_{n_i})$ appends patient $n_i$'s latent code $z_{n_i}$ to the survey vector $x_i$, which $\widehat{X}(Z)$ stacks row-wise into a design matrix. The kernel $K_\theta(\cdot, \cdot)$ operates on these augmented inputs with hyperparameters $\theta$. The vector $F$ collects latent function values $f_i$ across all observations, and $\mathbb{I}_{D_z}$ is the $D_z$-dimensional identity matrix.
The mean function $m_\theta(X) = [m_\theta(x_1), \dots, m_\theta(x_I)]^\intercal$ applies a code-independent transformation to input features $x_i$. Because SREs are low-prevalence outcomes, we set $m_\theta(x) = b$, where $b \in \theta$ is a learned scalar initialized to the logit of the empirical training prevalence.

\textbf{The Similarity Kernel.}
To ensure similarity between patients is well-posed, we constrain the kernel:
\begin{definition} \label{def:iskernel}
A kernel $K_\theta(\hat{x}, \hat{x}')$ on augmented inputs $\hat{x} = (x, z)$ and $\hat{x}' = (x', z')$ is a \emph{Similarity Kernel} if it factorizes as follows:
\begin{align}
    \underbrace{K^x_\theta(x, x')}_{\substack{\text{within-patient} \\ \text{covariance}}} \cdot \overbrace{\underbrace{\sigma_\theta(z)\,\sigma_\theta(z')}_{\text{amplitude}} \cdot \underbrace{\zeta_\theta(z, z')}_{\substack{\text{between-patient} \\ \text{similarity}}}}^{K^z_\theta(z, z')}, \label{eq:kz}
\end{align}
wherein:
\begin{enumerate}[nosep, label=(\alph*), font=\bfseries]
    \item $K^x_\theta(x, x') \not\equiv 0$ is not the degenerate 0 kernel.
    \item $\sigma_\theta: \mathbb{R}^{D_z} \to \mathbb{R}_{>0}$ is continuous.    
    \item $\zeta_\theta(z, z') := \rho_\theta\big(S \cdot (\varphi_\theta(z) - \varphi_\theta(z'))\big) \in [0, 1]$. 
\end{enumerate}
Within $\zeta_\theta(z, z')$, 
\begin{enumerate}[nosep, label=(\roman*), font=\bfseries]
    \item The \emph{correlation}, $\rho_\theta: \mathbb{R}^{D_z} \to [0, 1]$, is continuous, positive-definite, and peaked at the origin: $\rho_\theta(0) = 1$, and $\rho_\theta(v) < 1$ for $v \neq 0$.
    \item The \emph{importances}, $S = \mathrm{diag}(S_1, \dots, S_{D_z}) \in \theta$, are non-negative.
    \item The \emph{warp}, $\varphi_\theta: \mathbb{R}^{D_z} \to \mathbb{R}^{D_z}$, is continuous, injective, and order-preserving: $z \leq z'$ coordinatewise implies $\varphi_\theta(z) \leq \varphi_\theta(z')$ coordinatewise. 
\end{enumerate}
\end{definition}
In \cref{apx:instantiations}, we show that this framework encompasses the standard stationary kernels and describe new extensions.
By constraining the patient-code kernel to a unit-diagonal correlation matrix scaled by a learned amplitude \citep[cf.][]{heinonen2016non}, the LSGP adapts standard multi-output GP formulations \citep[e.g.,][]{bonilla2007kernel,dai2017efficient} to decompose the coregionalization matrix \citep{alvarez2012kernels} into two identifiable channels: trend and amplitude similarity (\cref{prop:identified}).

\textbf{Quantifying Patient Similarity.}
Evaluated at the patients' codes, Similarity Kernels decompose patient similarity into two $N \times N$ channels, a \emph{trend channel} $A^{\mathrm{trend}}_{nm} := \zeta_\theta(z_n, z_m)$ and an \emph{amplitude channel},
\begin{align*}
    A^{\mathrm{amp}}_{nm} &:= \tfrac{\min(\sigma_\theta(z_n),\, \sigma_\theta(z_m))}{\max(\sigma_\theta(z_n),\, \sigma_\theta(z_m))} = \exp\left(-\left\lvert\log \tfrac{\sigma_\theta(z_n)}{\sigma_\theta(z_m)}\right\rvert\right) ,
\end{align*}
both with entries $\in [0, 1]$ and with unit diagonal.
We evaluate these matrices at the posterior mode, $\operatorname*{arg\,max}_{Z}\, p(Z \mid \mathcal{D})$. To facilitate empirical evaluation, we also define their product, $A^{\mathrm{comb}}_{nm} := A^{\mathrm{trend}}_{nm} \cdot A^{\mathrm{amp}}_{nm} \in [0, 1]$, representing one natural combination of the two.

\textbf{The LSGP dynamically bridges idiographic and nomothetic regimes.}
\ul{Regime 1: Fully Nomothetic (\mbox{$S = 0$ and $\sigma_\theta(z) \equiv \mathrm{const}$}).} Here, $\zeta_\theta \equiv \rho_\theta(0) = 1$, reducing \cref{eq:kz} to $K^x_\theta(x, x')\,\sigma_\theta(z)\,\sigma_\theta(z') = K^x_\theta(x, x')\,\sigma_\theta^2$, resulting in a nomothetic GP.
\ul{Regime 2: Fully Disjoint (\mbox{$S \to \infty$}).}
Off-diagonal trend terms decay to zero ($A^{\mathrm{trend}}_{nm} \to 0$ for $n \neq m$), yielding a block-diagonal kernel corresponding to $N$ independent, idiographic GPs.
\ul{Regime 3: Adaptive Interpolation (\mbox{$0 < S < \infty$}).} 
The model dynamically shares statistical patterns across patients based on inferred similarity, learning the optimal degree of inter-patient sharing through $S$.

\textbf{Identifiability of Similarity.}
Without an identifiable kernel decomposition, scale and similarity parameters trade off arbitrarily, allowing high-amplitude patients to appear deceptively similar to the entire cohort. 
As formalized in \cref{prop:identified}, our kernel structure eliminates this ambiguity: 
while the overall covariance admits an unavoidable global scale factor, both $A^{\mathrm{trend}}$ and $A^{\mathrm{amp}}$ are strictly invariant, ensuring that inferred patient relationships represent stable structural properties of the data rather than parameterization artifacts.
\begin{proposition}[Proof in \cref{apx:identif}] \label{prop:identified}
Let $K_\theta$ and $K_{\theta'}$ be Similarity Kernels (\cref{def:iskernel}).
\textbf{Then} $K_\theta = K_{\theta'}$ \textbf{if and only if} there is a scalar $\lambda > 0$ such that, as functions,
\begin{align}
    K^x_\theta = \lambda\, K^x_{\theta'} , \qquad \sigma_{\theta'} = \sqrt{\lambda}\, \sigma_\theta , \qquad \zeta_\theta = \zeta_{\theta'} .
\end{align}
As a result, when $K_\theta = K_{\theta'}$, we have $A^{\mathrm{trend}} = A'^{\mathrm{trend}}$ and $A^{\mathrm{amp}} = A'^{\mathrm{amp}}$, whatever parameters or functional forms produced them.
\end{proposition}

\section{APPROXIMATE INFERENCE} \label{sec:pg}

\textbf{Sparse Variational LSGPs (SV-LSGP).}
Exact LSGP posterior inference is intractable due to the non-conjugate likelihood and non-linear latent variable transformations.
Prior work addresses this via a sparse variational approximation~\citep{titsias2009variational,hensman2013gaussian,hensman2015scalable}, as applied to LVMOGPs~\citep{dai2017efficient,hang2025improving}, conditioning on a small set of $M$ inducing point locations, $W \in \mathbb{R}^{M \times (D_x + D_z)}$, used to ``summarize'' the training inputs, replacing $\Circled{C}$ with:
\begin{align*}
    &\Circled{C.1}\ \ U; W, \theta \sim \mathcal{N}( 0, K_\theta(W, W) ), \\
    &\Circled{C.2}\ \ F \mid U, \widehat{X}(Z); W, \theta \sim \mathcal{N}\big( m_\theta(X) + \Psi(Z) U, \\
    &\qquad\qquad K_\theta(\widehat{X}(Z), \widehat{X}(Z)) - \Psi(Z) K_\theta(\widehat{X}(Z), W)^\intercal \big).
\end{align*}
Here $\Psi(Z) := K_\theta(\widehat{X}(Z), W) K_\theta(W,W)^{-1}$.
$W$ and $\theta$ are then typically learned with MFVI, i.e., with the variational family $q(Z, U; \phi) = q(U; \phi) \prod_{n=1}^N q(z_n; \phi)$~\citep[e.g.,][]{hensman2013gaussian,lalchand2022generalised} (details in \cref{apx:sv}).

\textbf{Pitfalls of MFVI.}
While computationally convenient, MFVI can be biased when local latent variables and global parameters share dependencies.
In Bayesian Neural Networks with Latent Variables~\citep{depeweg2018decomposition}---which share the LSGP's structure (local latent variables and global parameters)---the posterior mode is asymptotically biased, leading MFVI to underfit and miscalibrate uncertainty~\citep{yacoby2022mitigating}.
We therefore adapt the P\'olya-Gamma augmentation scheme~\citep{polson2013bayesian, wenzel2019efficient} to the LSGP. 
This restores conditional Gaussian conjugacy for $U$, yielding a structured variational family via closed-form coordinate-ascent updates, similar to structured stochastic variational inference for conditionally conjugate models~\citep{hoffman2015structured}, enabling the theoretical comparison of our inference method to MFVI (\cref{sec:bias}).

\textbf{P\'olya-Gamma LSGP (PG-LSGP).}
We write the likelihood as an expectation over an auxiliary variable $\omega_i$~\citep{polson2013bayesian},
\begin{align}
\begin{split}
    p(y_i \mid f_i) &= \mathbb{E}_{\omega_i \sim \mathrm{PG}(1, 0)}\!\left[\ell(y_i \mid f_i, \omega_i)\right], \\
    \log \ell(y_i \mid f_i, \omega_i) &= \tfrac{2y_i-1}{2}\, f_i - \tfrac{\omega_i}{2}\, f_i^2 - \log 2.
\end{split} \label{eq:log_ell_main}
\end{align}
Like \citet{wenzel2019efficient}, we introduce $q(\omega_i \mid Z)$ per observation and apply Jensen's inequality to obtain a variational bound,
\begin{align}
    \mathcal{L}_{\mathrm{PG-ELBO}} ={}& \sum_{i} \mathbb{E}_{q}\!\left[\log \ell(y_i \mid f_i, \omega_i)\right] \nonumber \\
    & - \sum_{i} \mathbb{E}_{q(Z)}\!\left[D_{\mathrm{KL}}[q(\omega_i \mid Z) \,\|\, \mathrm{PG}(1,0)]\right] \nonumber \\
    & - \mathbb{E}_{q(Z)}\!\left[D_{\mathrm{KL}}[q(U\mid Z) \,\|\, p(U)]\right] \nonumber \\
    & - D_{\mathrm{KL}}[q(Z) \,\|\, p(Z)] , \label{eq:pg_elbo}
\end{align}
where $\mathbb{E}_{q}$ is under the variational joint $q(Z;\phi)\,q(U\mid Z)\,p(f_i\mid U)\,q(\omega_i\mid Z)$ (derivation in \cref{apx:pg}).

\textbf{Variational Expectation-Maximization (VEM).}
For the PG-LSGP, coordinate-ascent updates over the conjugate variational factors admit exact closed forms:

\textbf{Step 1 (E-step, auxiliary update).} Update the local P\'olya-Gamma factors at the previous sweep's $c_i$:
\begin{align}
    q^{(t)}(\omega_i \mid Z) \leftarrow \mathrm{PG}\big(1,\, c^{(t-1)}_i\big) . \label{eq:pg_c}
\end{align}
\textbf{Step 2 (E-step, inducing-posterior update).} Update the global inducing-value factor, conditioned on $Z$:
\begin{align}
    q^{(t)}(U \mid Z) \leftarrow \mathcal{N}\big(\mu^{(t)}(Z),\, \Sigma^{(t)}(Z)\big) . \label{eq:pg_moments}
\end{align}
The moments for Steps 1 and 2 are calculated via:
\begin{align*}
    \bar\omega^{(t)}_i &= \tfrac{1}{2 c^{(t-1)}_i}\tanh\left(\tfrac{c^{(t-1)}_i}{2}\right), \\
    \bar\Omega^{(t)}(Z) &= \mathrm{diag}\big(\bar\omega^{(t)}_1, \dots, \bar\omega^{(t)}_I\big), \\
    \Sigma^{(t)}(Z)^{-1} &= K_\theta(W,W)^{-1} + \Psi(Z)^\intercal\, \bar\Omega^{(t)}(Z)\, \Psi(Z), \\
    \mu^{(t)}(Z) &= \tfrac12\,\Sigma^{(t)}(Z)\,\Psi(Z)^\intercal \big(2Y - 1 - 2\bar\Omega^{(t)}(Z)\, m_\theta(X)\big), \\
    c^{(t)}_i &= \sqrt{\mathbb{E}_{q^{(t)}(f_i \mid Z)}\big[f_i^2\big]} ,
\end{align*}
where $q^{(t)}(f_i \mid Z)$ is the marginal left by integrating $U \sim q^{(t)}(U \mid Z)$ out of $\Circled{C.2}$.

\textbf{Step 3 (M-step).} Update parameters $\{\theta, W, \phi\}$ via a stochastic gradient step, evaluated at the converged variational distribution $q^{(T)}$:
\begin{align}
    \{\theta, W, \phi\} \leftarrow \{\theta, W, \phi\} + \alpha\, \widehat{\nabla}\, \mathcal{L}_{\mathrm{PG-ELBO}}\big\rvert_{q^{(T)}} , \label{eq:pg_mstep}
\end{align}
where $\widehat{\nabla}$ is a Monte Carlo gradient estimator under reparameterized samples $Z \sim q(Z; \phi)$~\citep{kingma2013auto}.

\textbf{Training Scheme.} For each outer VEM step, Steps 1 and 2 are repeated for $T=15$ sweeps to achieve conditional convergence. 
With mini-batching and at unit step size, those conjugate updates are exactly natural-gradient steps~\citep{wenzel2019efficient}; Step 3 then performs a single, plain gradient M-step. 
Details in \cref{apx:methods}.

\textbf{Variational Family.}
Because $q(\omega_i \mid Z)$ and $q(U \mid Z)$ have closed-form solutions, we only need to parameterize $q(Z; \phi)$. 
We set it to a diagonal Gaussian $q(Z; \phi) = \prod_{n=1}^N \mathcal{N}\big(\mu_{z_n},\, \mathrm{diag}(\sigma^2_{z_n}) \big)$ with moments amortized via a shared encoder, $(\mu_{z_n}, \sigma^2_{z_n}) = g_\phi(\mathcal{D}_n)$. To accommodate variable-length, unordered datasets $\mathcal{D}_n$, we parameterize $g_\phi$ as a DeepSet \citep{zaheer2017deep}, embedding each point $(x, y) \in \mathcal{D}_n$, mean-pooling the embeddings, and decoding the aggregated vector. This guarantees permutation invariance by construction and mirrors conditional neural processes with patients as tasks and their surveys as context sets~\citep{garnelo2018conditional}.

\begin{figure*}[t]
    \centering
    \includegraphics[width=\textwidth]{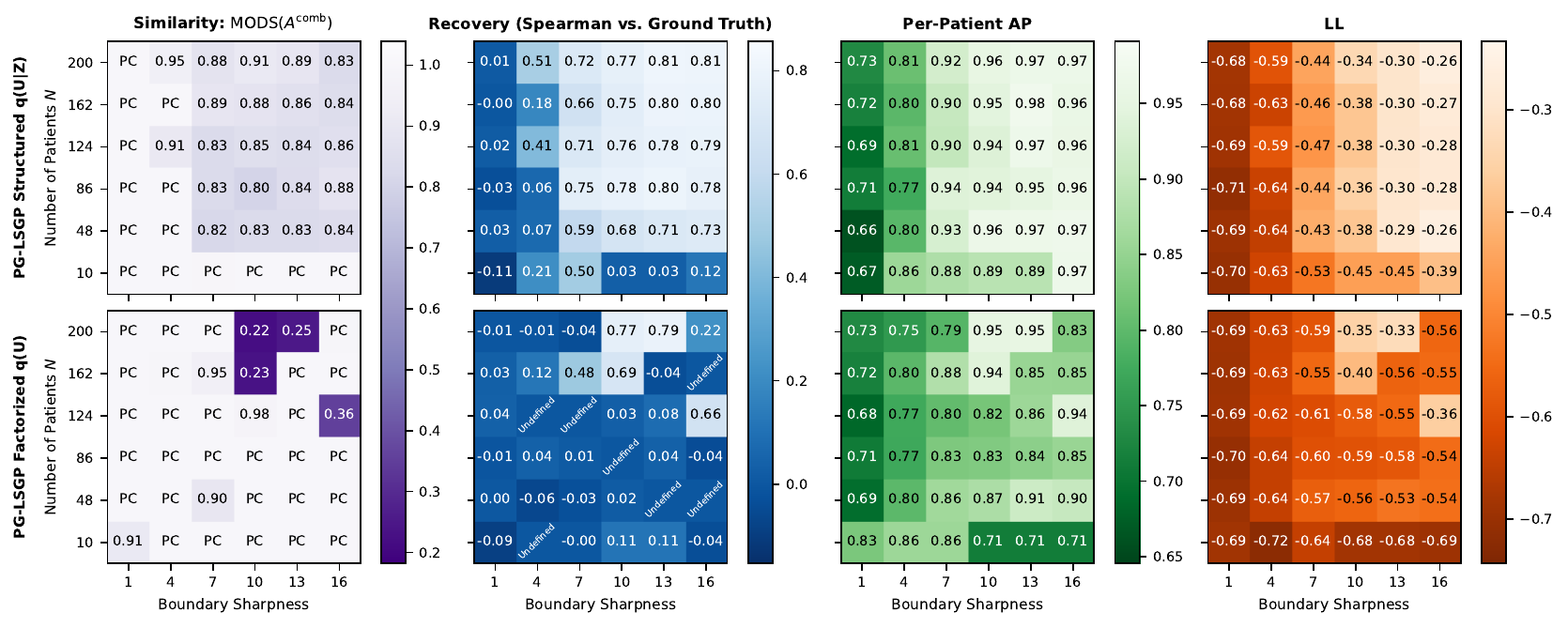}
    \caption{\textbf{Structured Outperforms Factorized PG-LSGP by Learning Patient Similarity.}
    \textbf{Top row:} structured.
    \textbf{Bottom row:} factorized.
    $\mathrm{MODS}$ is the mean off-diagonal similarity (\cref{eq:mods}); cells labeled \textsf{PC} denote personalization collapse ($\mathrm{MODS}(A^{\mathrm{comb}}) = 1.00$ to two decimals); cells labeled \textsf{Undefined} represent rank correlations that are undefined because the learned similarity matrix is constant off the diagonal.
    Recovery is the Spearman rank correlation between the learned and true similarity.
    Details in \cref{sec:synthetic,apx:data}.}
    \label{fig:gap-heatmap}
\end{figure*}

\section{PERSONALIZATION COLLAPSE} \label{sec:bias}

We now show that the commonly used factorization $q(Z, U) = q(Z) \cdot q(U)$ in MFVI~\citep[e.g.,][]{hensman2013gaussian,lalchand2022generalised} introduces a structural bias into the LSGP: it adds a penalty to the objective that vanishes only at a nomothetic model---that is, only at $A^\text{trend} = A^\text{amp} = 1$.
We call this pathology \emph{personalization collapse}.

\textbf{Characterizing the Coupling Gap.}
To quantify the cost of discarding the code-dependent posterior $q(U \mid Z)$, we formulate the variational gap as a Bregman information over natural parameters.

\begin{definition}[\cite{bregman1967relaxation}] \label{def:bregman}
For a strictly convex, differentiable $g$ on a convex domain, the \emph{Bregman divergence} it generates is $B_{g}(x, y) := g(x) - g(y) - \langle \nabla g(y),\, x - y\rangle \ge 0$, with equality if and only if $x = y$.
\end{definition}

\begin{definition}[\cite{banerjee2005clustering}] \label{def:breg_info}
For $x \sim q$ valued in the domain of $g$ (\cref{def:bregman}), the \emph{Bregman information} of $x$ is the least expected divergence to a single representative, $I_g(x) := \min_r \mathbb{E}_q[B_g(x, r)] = \mathbb{E}_q[g(x)] - g(\mathbb{E}_q[x]) \ge 0$, attained at the mean $r = \mathbb{E}_q[x]$~\citep{banerjee2005optimality} and vanishing if and only if $x$ is $q$-a.s.~constant.
\end{definition}

Let $\mathcal{A}$ denote the log-partition function of an $M$-dimensional multivariate Gaussian.
From \cref{eq:pg_moments}, the exact conditional variational posterior $q^*(U \mid Z) = \mathcal{N}(\mu(Z), \Sigma(Z))$ can be written in natural coordinates as $\eta(Z) := \big(\Sigma(Z)^{-1}\mu(Z),\, \Sigma(Z)^{-1}\big)$.

\begin{theorem}[Proof in \cref{apx:proof-gap}] \label{thm:gap}
\textbf{Suppose}:
\begin{enumerate}[nosep, label=(\roman*), font=\bfseries]
    \item The auxiliary factor $q(\Omega \mid Z) = \prod_i q(\omega_i \mid Z)$ and the code posterior $q(Z; \phi)$ are held fixed, so that the two maxima differ only in the $U$-factor;
    \item $\mathbb{E}_{q(Z;\phi)}\lVert\eta(Z)\rVert < \infty$ and $\mathbb{E}_{q(Z;\phi)}[\mathcal{A}(\eta(Z))] < \infty$; 
    \item And $q(U)$ and $q(U \mid Z)$ can attain their optima in \cref{eq:apx_gap} (automatic for full-covariance Gaussians).
\end{enumerate}
\textbf{Then}:
\begin{enumerate}[nosep, label=(\alph*), font=\bfseries]
    \item The coupling gap is the Bregman information of $\eta(Z)$ under $\mathcal{A}$:
    \begin{align}
        \mathrm{GAP} &:= \max_{q(U\mid Z)} \mathcal{L}_{\mathrm{PG-ELBO}} - \max_{q(U)} \mathcal{L}_{\mathrm{PG-ELBO}} \nonumber \\
        &= I_{\mathcal{A}}(\eta(Z)) \;\ge\; 0 . \label{eq:apx_gap}
    \end{align}
    \item And $\mathrm{GAP} = 0$ \textbf{if and only if} $\eta(Z)$ is $q(Z;\phi)$-a.s.\ constant.
\end{enumerate}
\end{theorem}
Thus, the optimal MFVI objective is the structured objective, penalized by GAP: at every $(\theta, W, \phi)$,
\begin{equation} \label{eq:mf_split}
    \max_{q(U)} \mathcal{L}_{\mathrm{PG-ELBO}} = \max_{q(U \mid Z)} \mathcal{L}_{\mathrm{PG-ELBO}} - I_{\mathcal{A}}(\eta(Z)).
\end{equation}
This exposes a fundamental tension in factorized inference: the structured objective rewards the code-dependence needed to express the data, while $I_{\mathcal{A}}(\eta(Z))$ penalizes it.
Next, we prove that driving $I_{\mathcal{A}}(\eta(Z)) \to 0$ results in a fully nomothetic model.

\textbf{Driving GAP$\to$0 Erases Patient Heterogeneity.}
Split each inducing location into input and latent blocks, $W = [\,W^x \mid W^z\,]$, so $w^x_m \in \mathbb{R}^{D_x}$ and $w^z_m \in \mathbb{R}^{D_z}$ are the two halves of inducing point $m$'s location.

\begin{theorem}[Proof in \cref{apx:proof-oversim}] \label{thm:oversim}
Call an inducing point $m$ ``live'' if $K^x_\theta(x_i, w^x_m) \neq 0$ for some $i$.
\textbf{Suppose} the assumptions of \cref{thm:gap} hold, and:
\begin{enumerate}[nosep, label=(\roman*), font=\bfseries]
    \item $W$ is fixed, and $K_\theta(W, W)$ is nonsingular (as the definition of $\Psi$ presupposes);
    \item $q(Z;\phi)$ has full support on $\mathbb{R}^{N D_z}$; 
    \item The auxiliary factor is held fixed at a \emph{code-independent} $q(\Omega) = \prod_i q(\omega_i)$, as in the fully factorized posterior $q(U)\,q(Z)\,q(\Omega)$, so its P\'olya-Gamma means form a constant diagonal $\bar\Omega^\star = \mathrm{diag}(\bar\omega_1, \dots, \bar\omega_I)$;
    \item $K_\theta$ is a Similarity Kernel (\cref{def:iskernel}); and
    \item For each $d \in \{1, \dots, D_z\}$ there are \emph{live} inducing points $m \neq m'$ with $\varphi_\theta(w^z_m)_d \neq \varphi_\theta(w^z_{m'})_d$.
\end{enumerate}
\textbf{Then}:
\begin{enumerate}[nosep, label=(\alph*), font=\bfseries]
    \item $\mathrm{GAP} = 0$ \textbf{if and only if} $S = 0$ \emph{and} $\sigma_\theta$ is constant.
    \item And whenever $\mathrm{GAP} = 0$, both similarity channels collapse to all-ones, $A^{\mathrm{trend}} = A^{\mathrm{amp}} = 1$, at \emph{any} codes, including the inferred posterior mode.
\end{enumerate}
\end{theorem}

\textbf{Enriching $q(Z)$ Alone Cannot Prevent Personalization Collapse.} 
Because the gap arises from enforcing conditional independence between $U$ and $Z$, \cref{thm:oversim} holds for \emph{any} non-degenerate, full-support $q(Z)$, including the \emph{exact} posterior $p(Z \mid \mathcal{D})$.

\textbf{The Trade-Off Between Data Fit and GAP.}
We conjecture that when the data carries little predictive signal---e.g.~when a cohort is small or its decision boundary is diffuse---the ELBO can be maximized by collapsing personalization ($S \to 0$, $\sigma_\theta \to \mathrm{const}$), setting $\mathrm{GAP} = 0$ at the cost of erasing patient dynamics. 
We demonstrate this empirically in \cref{sec:synthetic}.

\begin{figure*}[t]
    \centering
    \includegraphics[width=\textwidth]{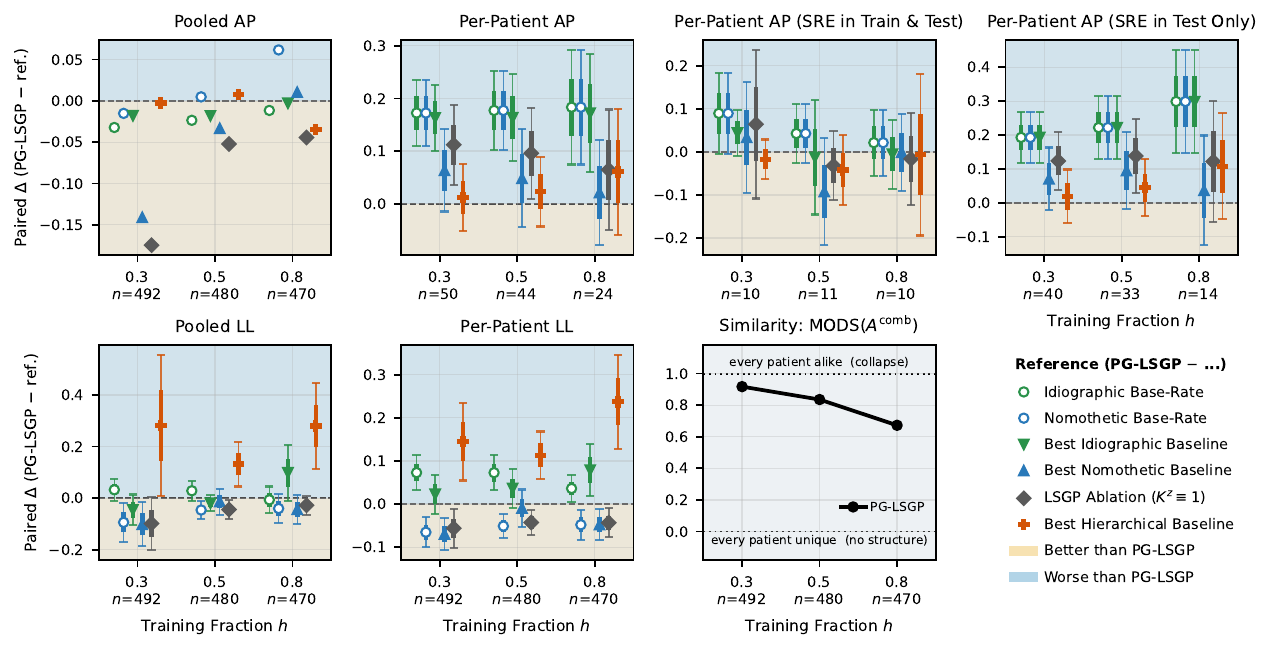}
    \caption{\textbf{Paired per-patient comparison: PG-LSGP outperforms baselines on per-patient AP.} Markers denote mean patient-level differences (LSGP minus reference); bars denote $\pm 1$ and $ \pm 2$ paired standard errors, for $n$ patients. Pooled AP (non-decomposable) is shown without bars, while the MODS panel confirms the LSGP avoids collapse across all $h$. The LSGP gains concentrate in test-only event cases (top-right panel).}
    \label{fig:q7}
\end{figure*}

\section{RESULTS: SYNTHETIC DATA} \label{sec:synthetic}

So far, we showed that the similarity an LSGP recovers is identifiable (\cref{prop:identified}) and that the factorized variational family carries a penalty that vanishes only when that similarity collapses (\cref{thm:gap,thm:oversim}). We now show that, on synthetic data with known ground truth, we do indeed recover the true similarity.

Note that we prove \cref{thm:oversim} for a code-independent auxiliary $q(\Omega)$ (assumption \textbf{(iii)}) and a fixed $W$ (assumption \textbf{(i)}); in the empirical analysis, we relax both assumptions---using the code-conditional $q(\Omega \mid Z)$ (derivation in \cref{apx:pg-factorized}) and freely learned inducing points---and \emph{still} observe personalization collapse.

\textbf{Setup.}
We evaluate $36$ synthetic cohorts from the ``Hemisphere'' process (\cref{apx:data-synthetic}), where unique patient classification boundaries are sampled on a 1D manifold. The cohorts consist of six cohort sizes ranging from $N = 10$ to $200$, crossed with six levels of steepness of the decision surface around the decision boundary ($\gamma = 1$ to $16$), called ``boundary sharpness.''
This evaluates the performance of methods across different amounts of predictive signal in the data.
Across 3 random restarts per cohort, we fit PG-LSGPs with $q(U \mid Z)$ vs.~$q(U)$, selecting the best of each across the restarts via validation log-likelihood (LL). We report Mean Off-Diagonal Similarity of $A^{\mathrm{comb}}$, $\mathrm{MODS}(A^{\mathrm{comb}})$, to assess collapse, ground-truth similarity recovery via the Spearman rank correlation between the learned and true similarity, and held-out LL and per-patient Average Precision (AP). Details in \cref{apx:data}.

\textbf{Result \ref{sec:synthetic}.1: Under Low Signal, the Structured Model Becomes Nomothetic.}
First, \cref{fig:gap-heatmap} confirms our conjecture (\cref{sec:bias}): when the likelihood provides minimal signal (e.g., small cohorts or diffuse decision boundaries), structured and factorized bounds perform comparably; both reduce to the nomothetic model. This adaptive behavior is desirable, as pooling data across patients is optimal under weak signal. 

\textbf{Result \ref{sec:synthetic}.2: Structured PG-LSGP Outperforms the Factorized Baseline by Learning Patient Similarity as Signal Strengthens.}
As shown in \cref{fig:gap-heatmap}, as the signal strengthens (higher $N$ and/or $\gamma$), the structured PG-LSGP quickly prevents personalization collapse, better recovers ground-truth patient similarity, and achieves better per-patient AP and LL.

\begin{figure*}[t]
    \centering
    \includegraphics[width=\textwidth]{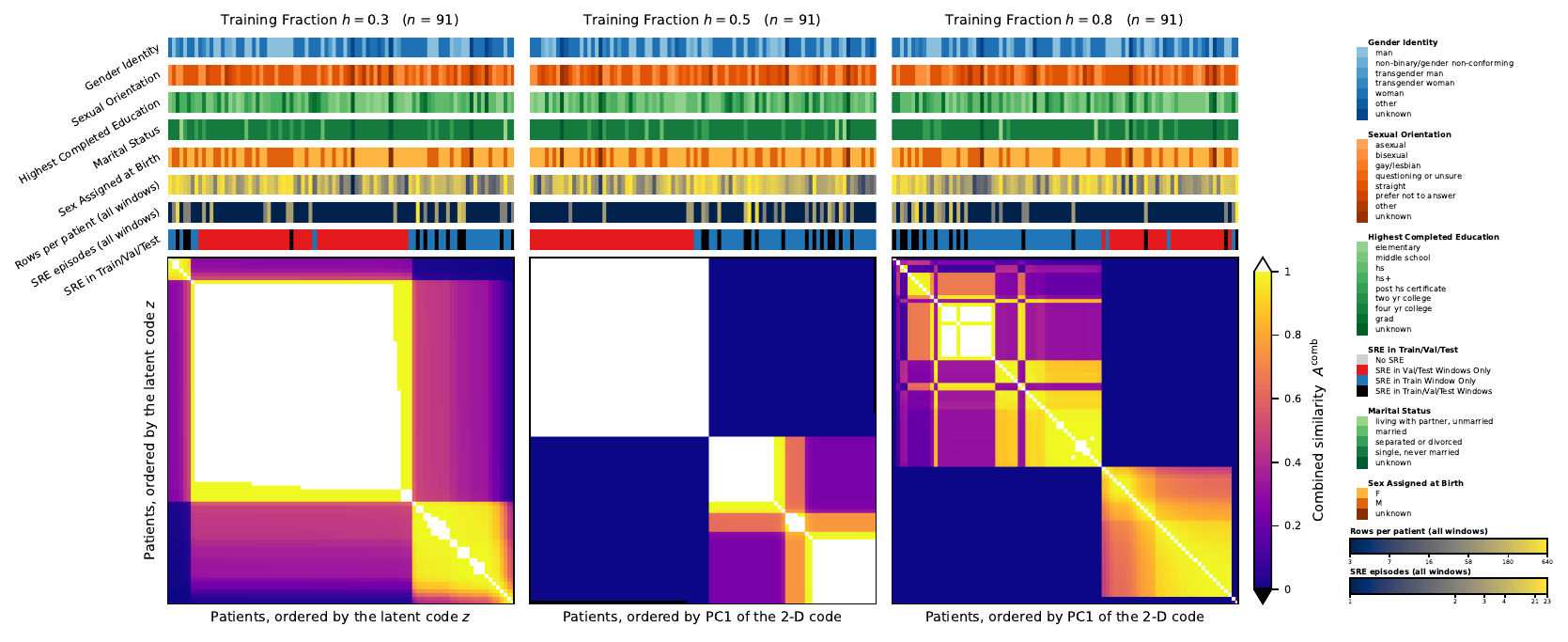}
    \caption{\textbf{Learned combined similarity ($A^{\mathrm{comb}}$) for patients with an SRE anywhere in their record.} Each panel shows the variant with the highest validation per-patient AP, at the restart selected on validation LL, for its $h$. Patients are ordered by the first principal component of their inferred latent codes $z_n$, with annotation strips displaying demographic categories, total survey counts, and numbers of SRE episodes. Full cohort in \cref{fig:heatmap-extended}.}
    \label{fig:heatmap}
\end{figure*}

\section{RESULTS: REAL DATA} \label{sec:results}

We now apply the LSGP to EMA data that tracks suicidal thoughts and behaviors.
We confirm that the LSGP is able to combine nomothetic trends with idiographic dynamics, leading to improved per-patient forecasting performance by inferring similarities.

\textbf{Setup.}
We evaluate on EMA data ($N = 497$ patients, $76{,}798$ surveys; 3-month, 6-times-daily protocol), retaining all patients ($81.7\%$ SRE-free) to mirror clinical deployment. Timelines are split across per-patient training fractions $h \in \{0.3, 0.5, 0.8\}$, with the remainder divided equally into validation and test windows. Performance is evaluated primarily via \emph{per-patient} AP (justification in Result~\ref{sec:results}.1), alongside pooled AP, log-likelihood (LL), and $\mathrm{MODS}(A^{\mathrm{comb}})$. We selected all hyperparameters via validation per-patient AP, using validation LL to select intra-model restarts (5 per variant) to assess convergence (\cref{apx:setup-u01}). Baselines include idiographic linear variants (logistic regression, variational Bayes, empirical Bayes), nomothetic linear/non-linear baselines (gradient-boosted trees, random forests), hierarchical (mixed-effects) logistic regression with random intercepts and/or slopes, base-rate predictors, and a similarity-ablated LSGP ($K^z_\theta \equiv 1$). Details in \cref{apx:data}.


\textbf{Result \ref{sec:results}.1: LL Fails to Select Models Capable of Forecasting Within-Patient Risk.}
As shown in \cref{fig:q7}, the Nomothetic Base-Rate baseline---which outputs a constant prediction equal to the training prevalence---is among the highest performers on pooled and per-patient LL across all $h$; however, it performs poorly on all remaining metrics because a constant predictor cannot capture individualized risk. Under severe class imbalance ($96\%$ event-free surveys), LL prioritizes static base-rate calibration over learning dynamic, within-patient risk trajectories. These observations are consistent with existing work that identifies shortcomings of LL for task-specific model selection~\citep{yao2019quality,deshpande2024are}. \ul{We therefore use validation per-patient AP for model selection, which remains robust under extreme class imbalance and avoids disproportionately weighting individuals with higher survey volume.}
We still use validation pooled LL to select among random restarts within each model, since LL confirms inference successfully converged according to our modeling assumptions---a necessary check until future work establishes better model selection metrics for this regime (see \cref{sec:future}).

\textbf{Result \ref{sec:results}.2: LSGP Improves Within-Patient Risk Forecasting.}
\Cref{fig:q7} shows that the LSGP outperforms all baselines in overall per-patient AP across every $h$. Comparing the LSGP to its similarity-ablated variant ($K^z_\theta \equiv 1$) isolates the performance benefit of learning patient similarity: the gain is on the ``SRE in Test Only'' subgroup---patients without training-period SREs who experience test-period events. On these patients, idiographic models collapse to the chance-level floor due to a lack of positive training signals, while nomothetic models ignore individual trajectories.

\textbf{Result \ref{sec:results}.3: Global vs.~Per-Patient Forecasting Trade-Offs.}
We hypothesize the trade-off between pooled and per-patient AP stems from our LSGP mean function, which lacks a patient-specific base rate. In future extensions, we hope to introduce a personalized prior mean and incorporate baseline prevalence into the kernel as another similarity channel, carefully adapted to maintain our identifiability guarantees.

\textbf{Result \ref{sec:results}.4: Recovering Uncollapsed Patient Similarity.}
Inspecting $A^{\mathrm{comb}}$ across the full cohort (\cref{fig:heatmap-extended}) and the subset of patients with at least one SRE (\cref{fig:heatmap}) reveals two distinct groups: those with and those without training-window SREs.
This is expected: without positive training labels to delineate individual decision boundaries, unobserved risk profiles remain less differentiated.
Crucially, however, the LSGP identifies similarities between some patients with and without training-window SREs---leveraging this similarity to improve forecasts for individuals whose first SRE occurs in the test window.
Ultimately, the learned similarity captures nuanced risk heterogeneity and may be clinically useful beyond predictive gains.

\textbf{Result \ref{sec:results}.5: Learned Similarity Is Not Explained by Demographics, Event Rate, or Survey Volume.}
In \cref{fig:heatmap}, patients are sorted by inferred latent codes $z_n$ alongside demographic annotations and survey counts. Visual inspection shows no demographic clustering. Following \citet{hang2025improving}, we formalize this by treating the off-diagonal similarity $A^{\mathrm{comb}}$ as a weighted graph (\cref{fig:graph}) and measuring its weighted modularity $Q$~\citep{newman2004analysis} against a label-permutation null distribution. We find no statistical evidence of demographic structuring. The learned similarity is not driven by demographics, event frequencies (where the idiographic base-rate baseline trails the LSGP across all $h$; \cref{fig:q7}), or survey completion volume. Thus, the latent space encodes within-patient dynamics.

\section{LIMITATIONS \& FUTURE WORK} \label{sec:future}
We set out to develop a model class extending established frameworks in clinical psychology that satisfies three core requirements: (i) accurate short-term forecasting, (ii) an explicit and interpretable formulation of patient similarity, and (iii) an inference scheme whose objective does not itself bias that similarity toward collapse. 
The LSGP meets these requirements: its within-patient forecasts improve over all baselines, with the largest gains in patients with no training-window SREs who go on to experience an SRE in the test window.
And it does this by inferring patient similarity, exposing relevant historical precedents that may prompt clinicians to investigate unmeasured contextual factors driving risk.
Still, several open problems remain.

\textbf{Training Objectives \& Metrics.}
Because LL fails to differentiate uninformative base-rate models from predictive models on the EMA data, we question its suitability for model selection and its impact on Bayesian inference.
In future work, we plan to investigate alternative training objectives and selection criteria that optimize PAC-Bayes generalization bounds \citep[e.g.,][]{germain2016pac,reeb2018learning} rather than relying on marginal likelihoods.

\textbf{Clinical Grounding Beyond Demographics.}
While static demographic traits do not account for the learned similarity structure, the latent geometry likely captures underlying, unobserved clinical dynamics. In future work, we hope to validate the learned similarities against richer clinical variables to assess whether they correspond to actionable clinical phenotypes that can support clinical decision-making.

\textbf{Accounting for Temporal Dynamics.}
The current LSGP evaluates risk cross-sectionally via latent codes $z_n$, omitting temporal dependencies such as survey autocorrelation. Because psychiatric risk is inherently dynamic, future extensions should explicitly incorporate temporal dynamics across observation timestamps.

\FloatBarrier

\ifpreprintflag

\subsubsection*{Acknowledgments}
Research reported in this publication was supported by the National Institute of Mental Health (NIMH) of the National Institutes of Health (NIH) under Award Numbers R01MH139486, U01MH116928, P50MH132766, and K23MH135211. The content is solely the responsibility of the authors and does not necessarily represent the official views of the National Institutes of Health. The authors are also grateful for funding from the Tommy Fuss Collaborative STAR Fund, the Chet and Will Griswold Suicide Prevention Fund, and a gift from an anonymous donor. 

\fi

\subsection*{AI use statement}

\paragraph{Approach.} We approached AI use in this project as an informal, first-person human-computer interaction experiment, treating our own experiences as a subject of inquiry. We carefully observed how interacting with AI affected our cognitive processes, using these observations to form, test, and refine hypotheses about when AI enhanced versus eroded our productivity, creativity, and sense of meaning in our work. We hope that this reflective approach can generate insights (useful for ourselves, and maybe later for others too) in deciding how, when, and why to use AI in research.

\paragraph{AI Use.} In this work, we used generative AI to assist with software development, writing, and preliminary mathematical derivations under strict oversight:

\begin{itemize}
\item \textbf{Implementation and Analysis.}
We wrote and verified the core model implementations ourselves, and used AI to check them against the relevant papers.
We used AI to help refactor and write the generic experimental pipeline (synthetic data generators, hyperparameter grid search, evaluation pipelines, unit tests, and figure generation) and to run diagnostic checks on experimental outputs, which helped us identify and correct reporting bugs.

\item \textbf{Manuscript and Literature.} We used AI to generate initial draft text from a detailed, granular outline that we wrote; we scrapped or heavily edited nearly all of it. We also used AI to polish prose and to verify related-work coverage. We authored the core arguments, narrative flow, and conceptual structure entirely ourselves.

\item \textbf{Theory.} We used AI for initial proof sketching. Because the AI-generated derivations contained errors or were incomplete, we subsequently rewrote, finished, and verified all proofs ourselves.

\item \textbf{Verification and Accountability.} We audited all AI-assisted text, claims, and code against our experimental logs. We computed every numerical result reported in this manuscript directly from raw experimental artifacts, not from AI-generated text and not from AI syntheses of the raw results. 
\end{itemize}

We take full responsibility for the integrity, accuracy, and content of this manuscript.

{
\bibliographystyle{plainnat}
\bibliography{references}
}

\clearpage
\appendix
\crefalias{section}{appendix}
\crefalias{subsection}{subappendix}
\crefalias{subsubsection}{subsubappendix}

\thispagestyle{empty}
\onecolumn
\aistatstitle{APPENDIX}

\addtocontents{toc}{\protect\setcounter{tocdepth}{2}}  
\tableofcontents

\section{RELATED WORK} \label{apx:related-work}

There are many modeling paradigms for learning latent structure from EMA data alongside a prediction, ranging from classical frameworks like linear mixed-effects models \citep[e.g.,][]{dora2024tutorial} to continuous-time and deep sequence models like Latent SDEs \citep[e.g.,][]{lu2026neural}. 
Rather than broadly comparing across all paradigms, we constrain our model selection using two key criteria. 

First, to ensure clinical accessibility, the model should represent a minimal departure from frameworks familiar to clinical researchers---nomothetic, idiographic, and mixed-effects models---while offering an explicit, interpretable latent structure. 
Second, because our primary goal is fixed-horizon forecasting (predicting next-week risk), framing the task as a classification problem cleanly isolates our core theoretical contributions---e.g., kernel identifiability, personalization collapse of MFVI, etc.---without the confounding computational and inference complexities of sequential models. This establishes a rigorous theoretical foundation that future work can readily extend to more complex sequential frameworks and continuous-time point processes.

These constraints point directly to the GPLVM family \citep[e.g.,][]{titsias2010bayesian,martens2019decomposing,dai2017efficient}---a GP extended with a continuous input latent variable. 
By focusing on this model class, we can cleanly illustrate our core claims regarding continuous patient-similarity representations for risk forecasting, while leaving extensions to sequential models to future work.

\textbf{Existing Modeling Approaches for EMA Data in Clinical Psychology.}
Most work on forecasting suicide risk from EMA data targets suicidal thoughts rather than SREs \citep[e.g.,][]{kleiman2017examination,czyz2023ecological,lei2023ecological,wang2023idiographic}. SREs are particularly challenging to predict due to severe class imbalance and acute stressors layered atop suicidal thoughts.
Historically, nomothetic and idiographic approaches have been common in clinical modeling research due to their intuitive presentation~\citep[e.g.,][]{barlow2009can,beltz2016bridging}. 
But since they make unfavorable trade-offs between pooling and personalization, EMA researchers have increasingly been turning to hierarchical (mixed-effects) models. 
Hierarchical models bridge the gap between nomothetic and idiographic paradigms by decomposing risk trajectories into a shared population trend (fixed effects) and patient-specific deviations (random effects)~\citep{dora2024tutorial}.
In \cref{sec:method,apx:hlm-function-space}, we show that LSGPs generalize both nomothetic and idiographic extremes, as well as hierarchical models, offering four key advantages:
\begin{enumerate}[leftmargin=*,itemsep=3pt,topsep=2pt]
\item \textbf{Explicit Patient Similarity:} Hierarchical models assume that individual patient deviations are completely independent. In contrast, LSGPs directly quantify similarity between patients and use it to improve risk forecasting of SREs.
\item \textbf{Targeted Information Sharing:} Hierarchical models shrink all patient-specific trends toward a single global population mean. For patients with no SREs during the training window, risk predictions default entirely to this generic average. In contrast, LSGPs selectively inform forecasts using the sub-population of patients who share similar dynamic risk trajectories, significantly improving risk predictions for individuals who experience their first SRE during the test period (Result~\ref{sec:results}.2).
That is, under an LSGP, the prediction for a given patient is informed directly by their own observed data, while the model selectively reweighs data from other patients based on the learned similarity of their predictive functions. Finally, under an LSGP, the shrinkage strength itself is per-patient, set by the code-dependent amplitude rather than by a single cohort-wide parameter.
\item \textbf{Non-linear Representation Learning:} Hierarchical random slopes enter linearly through a fixed design matrix. In contrast, LSGPs transform the latent code through any choice of kernel, enabling the model to learn arbitrarily complex, non-linear transformations of risk trajectories over time.
\item \textbf{Capacity Control under High Sparsity:} Fitting random-slope models requires estimating $D_x + 1$ individual parameters per patient directly from sparse longitudinal data. In our cohort ($D_x = 35$), where $81.7\%$ of patients exhibit zero events, estimating this many parameters per individual can lead to overfitting. In contrast, LSGPs map each patient into a low-dimensional latent space (1--2 dimensions in our experiments), replacing dozens of individual parameter estimates with a compact representation of patient similarity.
\end{enumerate}

\textbf{Existing Approaches for Learning Patient Similarity from EMA Data in Clinical Psychology.}
Methods for inferring patient similarity in EMA data generally fall into two categories: (1) feature-based clustering on summary statistics~\citep[e.g.,][]{kleiman2018digital}, and (2) parametric latent variable modeling~\citep[e.g.,][]{nagin2024recent,dora2024tutorial}.
Summary-statistic approaches collapse EMA histories into aggregate metrics (e.g., means or variances) independently of risk forecasting. This discards temporal nuance: two patients with identical summary statistics may exhibit vastly different risk dynamics---such as a steady decline versus extreme volatility---that imply fundamentally different near-term SRE risk. 
Parametric latent models structure patient heterogeneity but restrict representations to discrete subgroup categories or impose rigid functional forms that struggle to capture complex, non-linear EMA patterns.
In contrast, we frame next-week risk prediction by modeling patient similarity directly within the predictive function space via a continuous latent process. This enables sharing statistical strength across patients with similar non-linear profiles without collapsing EMA histories into aggregate metrics, forcing discrete subtyping, or assuming restrictive parametric curves.

\textbf{Leveraging Similarity in Clinical Decision Support.}
Clinical decision support systems increasingly leverage patient similarity to retrieve historical precedents for case-based reasoning~\citep{sharafoddini2017patient}. 
Existing approaches broadly fall into two categories. 
The first calculates similarity over static health records using fixed distance metrics or domain heuristics \citep[e.g.,][]{sharafoddini2017patient,fang2021patient,oei2021using}.
The second uses metric learning or deep architectures to map longitudinal feature trajectories into latent embeddings on which similarities are computed \citep[e.g.,][]{che2017rnn,suo2018multi,zhu2019measuring,landi2020deep}. 
However, both paradigms evaluate similarity between two patients, $i$ and $j$, in the input feature space ($x_i \approx x_j$) or latent embedding space ($z_i \approx z_j$), rather than in function space ($f_i \approx f_j$). 
For dynamic outcomes like suicide risk, matching feature values or embedding overlaps is insufficient.
Function-space similarity instead captures \emph{how} the inputs, $x$---like daily EMA survey responses---dynamically drive changes in risk $y = f(x)$. That is, two patients may present with identical symptom scores ($x_i \approx x_j$), yet these scores may drive suicide risk differently ($f_i(x_i) \neq f_j(x_j)$): for example, an acute spike in interpersonal conflict might increase suicide risk in one patient while leaving another relatively stable. 
The LSGP addresses these limitations by explicitly capturing function-space similarity while providing formal identifiability guarantees.

\textbf{Inference Pathologies in GPLVM Family Models.}
Scalable inference for the GPLVM family typically relies on MFVI, factorizing the variational posterior across the inducing variables and the local latent variables~\citep[e.g.,][]{hensman2013gaussian,lalchand2022generalised} (the collapsed bound of \citet{titsias2010bayesian} instead integrates the inducing variables out analytically, which is possible for a Gaussian likelihood but not for the Bernoulli likelihood here). 
However, the specific inference pathologies arising from MFVI in GPLVMs remain uncharacterized, despite a broad literature showing pathologies of MFVI in other models \citep[e.g.,][]{turner2011two,foong2020expressiveness,yacoby2020failure,coker2022wide}. 
Most closely related to the pathology we expose here is the work by \citet{yacoby2022mitigating} on Bayesian Neural Networks with Latent Variables (BNN+LVs)---the neural network analog of GPLVMs. 
In BNN+LVs, MFVI drives inference toward an asymptotically biased posterior mode, resulting in poor model fit, generalization, and uncertainty quantification. 
While this pathology has been demonstrated in neural architectures, its impact on non-parametric GP-based models has not been characterized.
In this work, we give an exact, closed-form characterization of the penalty incurred by MFVI for LSGPs and show that minimizing it collapses the LSGP to a nomothetic model.
In future work, we hope to translate these findings from the LSGP to the broader GPLVM family.

\section{HIERARCHICAL MODELS AS GAUSSIAN PROCESSES} \label{apx:hlm-function-space}

We now recall the well-known result that hierarchical (mixed-effects) models can be represented as GPs, using this relationship to highlight the exact points where LSGPs differ.

\textbf{Hierarchical (Mixed-Effects) Model.}
Let $\beta_0 \in \mathbb{R}$ and $\beta \in \mathbb{R}^{D_x}$ denote the shared fixed-effect bias and slopes.
Each patient $n$ carries a random bias $z_{n,0} \in \mathbb{R}$ and, in the \emph{intercept+slopes} model, random slopes $z_{n,1:D_x} := (z_{n,1}, \dots, z_{n,D_x}) \in \mathbb{R}^{D_x}$; the \emph{intercept} model omits the random slopes, so all formulas below specialize to it by deleting the slope terms.
We write $z_n$ for patient $n$'s random effects and $\Sigma_\tau := \operatorname{diag}(\tau_1^2, \dots, \tau_{D_x}^2)$ for the slope variance components. The generative process is:
\begin{align}
\beta_0 &\sim \mathcal{N}(0, \sigma_{\mathrm{fe}}^2); \qquad \beta \sim \mathcal{N}(0, \sigma_{\mathrm{fe}}^2\, \mathbb{I}_{D_x}); \\
\log \tau_k &\sim \mathcal{N}(0, \sigma_{\mathrm{vc}}^2), && k = 0, \dots, D_x; \\
z_{n,0} &\sim \mathcal{N}(0, \tau_0^2); \qquad z_{n,1:D_x} \sim \mathcal{N}(0, \Sigma_\tau), && n = 1, \dots, N; \\
y_i \mid \beta, z &\sim \mathrm{Bernoulli}\Big( \mathrm{sigmoid}\big( f_{n_i}(x_i) \big) \Big),    
\end{align}
where the logit function for patient $n$ makes the shared trend and the patient's deviation explicit:
\begin{equation}
    f_n(x) := \underbrace{\beta_0 + \beta^\top x}_{\text{fixed effects}} + \underbrace{z_{n,0} + z_{n,1:D_x}^\top x}_{\text{random effects}} . \label{eq:hlm-weight-space}
\end{equation}

\textbf{Hierarchical Model as a GP.}
Integrating out $(\beta_0, \beta)$ and $z$ yields a zero-mean GP over the joint patient-feature index $(n, x)$:
\begin{align}
    \mathrm{Cov}\big(f_n(x), f_m(x')\big)
    &= \mathbb{E}[\beta_0^2] + x^\top \mathbb{E}[\beta \beta^\top]\, x' + \mathbb{E}[z_{n,0}\, z_{m,0}] + x^\top \mathbb{E}\big[ z_{n,1:D_x}\, z_{m,1:D_x}^\top \big]\, x' \nonumber \\
    &= \underbrace{\sigma_{\mathrm{fe}}^2 \cdot (1 + x^\top x')}_{k_{\mathrm{fixed}}(x, x')} \;+\; \mathbb{I}(n = m) \cdot \underbrace{\big( \tau_0^2 + x^\top \Sigma_\tau\, x' \big)}_{k_{\mathrm{random}}(x, x')} . \label{eq:hlm-function-space}
\end{align}
Thus, hierarchical models can be represented as GPs with a linear fixed-effect kernel $k_{\mathrm{fixed}}$ plus an indicator-gated linear deviation kernel $k_{\mathrm{random}}$. 

\textbf{LSGPs Generalize the Hierarchical Model.} 
Specifically, they do this by (1) replacing the rigid indicator $\mathbb{I}(n = m)$ with a learned continuous similarity kernel $K^z_\theta(z_n, z_m)$; (2) replacing the fixed identity input $n$ with an inferable low-dimensional code $z_n$; (3) relaxing linear feature kernels to flexible input kernels $K^x_\theta$; and (4) parameterizing shrinkage strength via a learned, patient-specific amplitude function $\sigma_\theta(z_n)^2$, parameterized by a neural network.

\section{THEORY FOR THE MODEL} \label{apx:theory-model}

\subsection{Instantiating the Similarity Kernel} \label{apx:instantiations}

\textbf{Existing Kernels.}
Recall from \cref{def:iskernel} that the similarity factor is defined as $\zeta_\theta(z,z') = \rho_\theta(r)$, where $r := \left\lVert S \big(\varphi_\theta(z) - \varphi_\theta(z')\big)\right\rVert$ is the weighted feature distance. We incorporate Automatic Relevance Determination (ARD) across all kernels via feature ``importances'' $S = \operatorname{diag}(S_1, \dots, S_{D_z}) \in \theta$. The table below shows that many popular kernels are instantiations of this similarity formulation, and each row can be combined with any $z$-independent kernel $K^x_\theta$ satisfying \cref{def:iskernel}.

\textbf{New Kernels.}
We extend the stationary, constant-amplitude kernels above by introducing two key mechanisms: (i) input-dependent amplitude, $\sigma_\theta(z) = \operatorname{softplus}(\operatorname{NN}_\theta(z))$, where $\operatorname{NN}_\theta$ is a neural network; and (ii) non-stationarity, via per-dimension input warping \citep{snoek2014input}. As shown by \citet{archer1993application}, a sigmoid network with non-negative weights yields a monotonic transformation. Whereas \citet{archer1993application} enforce non-negativity dynamically during optimization by damping learning rates, we enforce monotonicity \emph{by construction} via a softplus parameterization:
\begin{equation}
    \operatorname{warp}_\theta(z)_d := \operatorname{softplus}(a_d)\, z_d + \sum_{k=1}^{K} \operatorname{softplus}(w_{dk}) \operatorname{sigmoid}\Big(\operatorname{softplus}(b_{dk})\, z_d + c_{dk}\Big), \label{eq:warp}
\end{equation}
where $K$ is the number of sigmoid units per coordinate, and $a \in \mathbb{R}^{D_z}, w, b, c \in \mathbb{R}^{D_z \times K}$ are parameters in $\theta$. The bias terms $c$ remain unconstrained as they merely set activation thresholds. Crucially, this guaranteed positivity ensures that $\varphi_\theta = \operatorname{warp}_\theta$ is strictly order-preserving, satisfying \cref{def:iskernel}\textbf{(iii)}.

\begin{center}
\setlength{\aboverulesep}{4pt}
\setlength{\belowrulesep}{4pt}
\begin{tabular}{@{}cllll@{}}
\toprule
& \textbf{ARD $K^z_\theta$} & \textbf{Amplitude} & \textbf{Correlation} & \textbf{Warp} \\
\specialrule{\lightrulewidth}{4pt}{1.5pt}
\specialrule{\lightrulewidth}{0pt}{4pt}
\multirow{3}{*}{\rotatebox[origin=c]{90}{\scriptsize \shortstack{\textbf{Examples of}\\\textbf{Existing Kernels}}}}
 & RBF & $\sigma_\theta(z) \equiv \sigma \in \theta$ & $\rho_\theta(r) = e^{-r^2/2}$ & $\varphi_\theta(z) = z$ \\
\cmidrule(l){2-5}
 & Mat\'ern ($5/2$) & $\sigma_\theta(z) \equiv \sigma \in \theta$ & $\rho_\theta(r) = \big(1 + \sqrt{5}r + \tfrac{5}{3}r^2\big)e^{-\sqrt{5}r}$ & $\varphi_\theta(z) = z$ \\
\cmidrule(l){2-5}
 & \stackcell{Rational-\\Quadratic} & $\sigma_\theta(z) \equiv \sigma \in \theta$ & \stackcell{$\rho_\theta(r) = (1 + r^2/2\alpha)^{-\alpha}$\\where $\alpha \in \theta$} & $\varphi_\theta(z) = z$ \\
\specialrule{\lightrulewidth}{6pt}{6pt}
\multirow{3}{*}{\rotatebox[origin=c]{90}{\scriptsize \shortstack{\textbf{Examples of}\\\textbf{New Kernels}}}}
 & \stackcell{RBF,\\amplitude-scaled} & $\sigma_\theta(z) = \mathrm{softplus}(\mathrm{NN}_\theta(z))$ & $\rho_\theta(r) = e^{-r^2/2}$ & $\varphi_\theta(z) = z$ \\
\cmidrule(l){2-5}
 & \stackcell{RBF,\\per-dim.\ warp} & $\sigma_\theta(z) \equiv \sigma \in \theta$ & $\rho_\theta(r) = e^{-r^2/2}$ & $\varphi_\theta(z) = \mathrm{warp}_\theta(z)$ \\
\cmidrule(l){2-5}
 & \stackcell{RBF,\\per-dim.\ warp,\\amplitude-scaled} & $\sigma_\theta(z) = \mathrm{softplus}(\mathrm{NN}_\theta(z))$ & $\rho_\theta(r) = e^{-r^2/2}$ & $\varphi_\theta(z) = \mathrm{warp}_\theta(z)$ \\
\bottomrule
\end{tabular}
\end{center}

\subsection{The Identifiability of the Similarity} \label{apx:identif}

\begin{proof}[Proof of \cref{prop:identified}]
Write both kernels in the form of \cref{eq:kz}. \emph{($\Leftarrow$)} Assume the three identities hold for some $\lambda > 0$; upon substituting them, the three factors cancel:
\begin{align}
    K_{\theta'}(\hat{x}, \hat{x}')
    &= \tfrac{1}{\lambda} K^x_\theta(x, x') \cdot \sqrt{\lambda}\,\sigma_\theta(z) \cdot \sqrt{\lambda}\,\sigma_\theta(z') \cdot \zeta_\theta(z, z') \nonumber \\
    &= K_\theta(\hat{x}, \hat{x}') .
\end{align}
\emph{($\Rightarrow$)} Fix $(x_1, x_2)$ with $K^x_\theta(x_1, x_2) \neq 0$ and write $\lambda(z) := \sigma_{\theta'}(z)^2 / \sigma_\theta(z)^2$. Then, for all $x, x'$ and all $z, z'$,
\begin{align}
    K_\theta &= K_{\theta'} && \text{assumed} \label{eq:apx_hyp} \\
    \Longrightarrow \ K^x_\theta(x, x')\, \sigma_\theta(z)^2 &= K^x_{\theta'}(x, x')\, \sigma_{\theta'}(z)^2 && \text{at } z' = z, \ \zeta_\theta(z,z) = 1 \label{eq:apx_within} \\
    \Longrightarrow \ K^x_\theta &= \lambda(z)\, K^x_{\theta'} && \sigma_\theta > 0, \ \cref{def:iskernel}\textbf{(b)} \\
    \Longrightarrow \ \lambda(z) &= K^x_\theta(x_1, x_2) / K^x_{\theta'}(x_1, x_2) \ \equiv \ \lambda && K^x_\theta \text{ independent of } z, \ \cref{def:iskernel}\textbf{(a)} \\
    \Longrightarrow \ \sigma_{\theta'} &= \sqrt{\lambda}\, \sigma_\theta && \text{positive roots} \\
    \Longrightarrow \ \zeta_\theta &= \zeta_{\theta'} && \text{divide \cref{eq:apx_hyp} at } (x_1, x_2) \text{ by } K^x_\theta\, \sigma_\theta\, \sigma_\theta
\end{align}
Thus, evaluating at any pair of codes $z_n, z_m$, we have $A^{\mathrm{trend}} = A'^{\mathrm{trend}}$.
Finally, since $\sqrt{\lambda} > 0$ cancels out,
\begin{align}
    A'^{\mathrm{amp}}_{nm}
    = \frac{\min\big(\sqrt{\lambda}\,\sigma_\theta(z_n),\; \sqrt{\lambda}\,\sigma_\theta(z_m)\big)}{\max\big(\sqrt{\lambda}\,\sigma_\theta(z_n),\; \sqrt{\lambda}\,\sigma_\theta(z_m)\big)}
    = \frac{\min\big(\sigma_\theta(z_n),\; \sigma_\theta(z_m)\big)}{\max\big(\sigma_\theta(z_n),\; \sigma_\theta(z_m)\big)}
    = A^{\mathrm{amp}}_{nm} .
\end{align}
\end{proof}

\thispagestyle{empty}

\section{THEORY FOR APPROXIMATE INFERENCE} \label{apx:methods}

\subsection{The Sparse Variational LSGP Lower Bound} \label{apx:sv}

\textbf{Sparse Augmentation.}
We use the sparse augmentation $\Circled{C.1}$--$\Circled{C.2}$ of \cref{sec:pg}, whose $\Psi(Z)$ has rows $\psi_i(z_{n_i})$.
The inducing values $U$ carry the \emph{zero-mean} sparse GP, with the mean function entering additively at $F$; integrating $U$ out of $\Circled{C.1}$--$\Circled{C.2}$ therefore returns $\Circled{C}$ of \cref{sec:method} exactly.

\textbf{Inference.}
We define our variational family,
\begin{align}
    q(F, U, Z; W, \theta, \phi) &= p(F \mid U, \widehat{X}(Z); W, \theta) \cdot \underbrace{q(U \mid Z; \phi)}_{\mathcal{N}(\mu_\phi(Z),\Sigma_\phi(Z))} \cdot \; q(Z; \phi)
\end{align}
Here $q(U \mid Z; \phi)$ is Gaussian and $q(Z; \phi)$ is a general code posterior, both drawn from the families in the \emph{Variational Families} paragraph below; the mean-field default takes $q(Z; \phi) = \prod_{n=1}^N q(z_n; \phi)$ and $q(U \mid Z; \phi)$ independent of the codes.
The KL divergence between the approximate and true posteriors is:
\begin{align*}
    0 &\leq  
    D_{\mathrm{KL}}\!\left[ q(F, U, Z; W, \theta, \phi) \,\Big\|\, p(F, U, Z \mid X, Y; W, \theta)\right] \\
    &= D_{\mathrm{KL}}\!\left[p(F \mid U, \widehat{X}(Z); W, \theta) \cdot q(U \mid Z; \phi) \cdot q(Z; \phi) \,\Big\|\, p(F, U, Z \mid X, Y; W, \theta)\right] \\
    &= \mathbb{E}_{p(F \mid U, \widehat{X}(Z); W, \theta)\, q(U \mid Z; \phi)\, q(Z; \phi)}\!\left[\log \frac{p(F \mid U, \widehat{X}(Z); W, \theta) \cdot q(U \mid Z; \phi) \cdot q(Z; \phi)}{p(Y \mid F) \cdot p(F \mid U, \widehat{X}(Z); W, \theta) \cdot p(U; W, \theta) \cdot p(Z)}\right] + \log p(Y \mid X; W, \theta)
\end{align*}
Thus,
\begin{align}
    \log p(Y \mid X; W, \theta) &\geq
    \mathbb{E}_{p(F \mid U, \widehat{X}(Z); W, \theta)\, q(U \mid Z; \phi)\, q(Z; \phi)}\!\left[\log p(Y \mid F)\right] \nonumber \\
    &\quad - \mathbb{E}_{q(Z; \phi)}\!\left[D_{\mathrm{KL}}\!\left[q(U \mid Z; \phi) \,\|\, p(U; W, \theta)\right]\right] - D_{\mathrm{KL}}\!\left[q(Z; \phi) \,\|\, p(Z)\right] \\
    &= \sum_{i=1} \mathbb{E}_{p(f_i \mid U, \hat{X}(Z); W, \theta)\, q(U \mid Z; \phi)\, q(Z; \phi)}\!\left[\log p(y_i \mid f_i)\right] \nonumber \\
    &\quad - \mathbb{E}_{q(Z; \phi)}\!\left[D_{\mathrm{KL}}\!\left[q(U \mid Z; \phi) \,\|\, p(U; W, \theta)\right]\right] - D_{\mathrm{KL}}\!\left[q(Z; \phi) \,\|\, p(Z)\right] \\
    &= \sum_{i=1} \mathbb{E}_{q(f_i \mid \hat{X}(Z); W, \theta)\, q(Z; \phi)}\!\left[\log p(y_i \mid f_i)\right] \nonumber \\
    &\quad - \mathbb{E}_{q(Z; \phi)}\!\left[D_{\mathrm{KL}}\!\left[q(U \mid Z; \phi) \,\|\, p(U; W, \theta)\right]\right] - D_{\mathrm{KL}}\!\left[q(Z; \phi) \,\|\, p(Z)\right] \\
    &=: \mathcal{L}_{\text{SV-ELBO}}
\end{align}
wherein we follow \citet{hensman2015scalable} for the analytical form of:
\begin{align}
q(f_i \mid Z, X; W, \theta) 
    &= \mathbb{E}_{q(U \mid Z; \phi)}\!\left[ p(f_i \mid U, \hat{x}_i(z_{n_i}); W, \theta) \right] \\
    &= \mathcal{N}\!\left( m_\theta(x_i) + \psi_i(z_{n_i}) \cdot \mu_\phi(Z), \; K_\theta(\hat{x}_i(z_{n_i}), \hat{x}_i(z_{n_i})) + \psi_i(z_{n_i}) \cdot (\Sigma_\phi(Z) - K_\theta(W, W)) \cdot \psi_i(z_{n_i})^\intercal \right)
\end{align}
with $\psi_i(z_{n_i}) = K_\theta(\hat{x}_i(z_{n_i}), W) \cdot K_\theta(W, W)^{-1}$.
Mini-batching the reconstruction term brings the cost of a gradient step down to $O(M^3)$~\citep{hensman2015scalable,hensman2013gaussian,lalchand2022generalised}.

\textbf{Variational Families.}
For the inducing posterior $q(U \mid Z; \phi)$, we compare two designs: a \emph{factorized} Gaussian $\mathcal{N}(\mu_\phi, \Sigma_\phi)$ that treats $U$ as independent of the codes, and a \emph{structured} Gaussian $\mathcal{N}(\mu_\phi(Z), \Sigma_\phi(Z))$ whose moments depend directly on $Z$. To map the unordered set of codes $Z = \{z_n\}_{n=1}^N$ to the inducing moments while preserving permutation invariance, we parameterize this structured family using a DeepSet \citep{zaheer2017deep}.
For the latent code posterior $q(Z; \phi)$, we can leverage \emph{amortization} via DeepSets rather than optimizing free per-instance parameters. Specifically, $q(Z; \phi)$ uses an encoder $g_\phi$ evaluated over patient $n$'s observation set $\mathcal{D}_n$ (\cref{sec:pg}).

\subsection{The P\'{o}lya-Gamma Sparse Variational LSGP Lower Bound} \label{apx:pg}

\textbf{P\'{o}lya-Gamma Augmentation.}
From~\citet{polson2013bayesian}, we know that, for any $\alpha \in \mathbb{R}$:
\begin{align}
    \mathrm{sigmoid}(\alpha) 
    &= \frac{1}{1 + \exp(-\alpha)}
    = \frac{\exp(\tfrac{1}{2} \cdot \alpha)}{2 \cdot \cosh(\tfrac{1}{2} \cdot \alpha)}
    = \frac{1}{2} \cdot \mathbb{E}_{\omega \sim \mathrm{PG}(1, 0)} \left[ \exp\left( \frac{\alpha - \alpha^2 \cdot \omega}{2} \right) \right], \label{eq:pg_identity}
\end{align}
where $\mathrm{PG}(\cdot, \cdot)$ is the P\'{o}lya-Gamma distribution.
With this identity, we rewrite the Bernoulli likelihood as follows:
\begin{align}
p(y_i | f_i) 
    &= \mathrm{Bernoulli}\left( y_i \mid \mathrm{sigmoid}(f_i) \right) \\
    &= \mathrm{sigmoid}(f_i)^{y_i} \cdot \left(1 - \mathrm{sigmoid}(f_i) \right)^{1-y_i} \\
    &= \left(\frac{\exp(f_i)}{1 + \exp(f_i)}\right)^{y_i} \cdot \left(\frac{1}{1 + \exp(f_i)}\right)^{1-y_i} \\
    &= \frac{\exp(y_i f_i)}{(1 + \exp(f_i))^{y_i}} \cdot \frac{1}{(1 + \exp(f_i))^{1-y_i}} \\
    &= \frac{\exp(y_i f_i)}{(1 + \exp(f_i))^{y_i + 1 - y_i}} \\
    &= \frac{\exp(y_i f_i)}{1 + \exp(f_i)} \\
    &= \exp(y_i \cdot f_i) \cdot \mathrm{sigmoid}(-f_i) \\
    &= \mathbb{E}_{\omega_i \sim \mathrm{PG}(1, 0)} \underbrace{\left[ \frac{1}{2} \cdot \exp(y_i \cdot f_i) \cdot \exp\left( \frac{-f_i - f_i^2 \cdot \omega_i}{2} \right) \right]}_{\ell(y_i | f_i, \omega_i)} 
\end{align}
wherein $\log \ell(y_i | f_i, \omega_i)$ can be further simplified:
\begin{align}
    \log \ell(y_i \mid f_i, \omega_i) 
    &= \log\left(\frac{1}{2}\right) + y_i f_i + \frac{-f_i - f_i^2 \omega_i}{2} \\
    &= -\log 2 + y_i f_i - \frac{f_i}{2} - \frac{\omega_i}{2} f_i^2 \\
    &= -\log 2 + \frac{2y_i - 1}{2} f_i - \frac{\omega_i}{2} f_i^2 \label{eq:log_ell}
\end{align}

\textbf{The P\'{o}lya-Gamma Augmented ELBO.}
We define our variational family,
\begin{align}
    q(F, U, Z; W, \theta, \phi) &= p(F \mid U, \widehat{X}(Z); W, \theta) \cdot q(U \mid Z) \cdot \underbrace{q(Z; \phi)}_{\prod_{n=1}^N q(z_n; \phi)},
\end{align}
and begin with the same ELBO, $\mathcal{L}_{\text{SV-ELBO}}$, of the Sparse Variational LSGP.
Then, drawing on work by~\citet{wenzel2019efficient}, we substitute the PG identity $p(y_i \mid f_i) = \mathbb{E}_{\omega_i \sim \mathrm{PG}(1,0)}\!\left[\ell(y_i \mid f_i, \omega_i)\right]$ into $\mathcal{L}_{\text{SV-ELBO}}$:
\begin{align}
    \mathcal{L}_{\text{SV-ELBO}}
    &= \sum_{i=1} \mathbb{E}_{q(Z; \phi)\, q(U \mid Z)\, p(f_i \mid U, \hat{x}_i(z_{n_i}); W, \theta)}\!\left[\log p(y_i \mid f_i)\right] \nonumber \\
    &\quad - \mathbb{E}_{q(Z; \phi)}\!\left[D_{\mathrm{KL}}\!\left[q(U \mid Z) \,\|\, p(U; W, \theta)\right]\right] - D_{\mathrm{KL}}\!\left[q(Z; \phi) \,\|\, p(Z)\right] \\
    &= \sum_{i=1} \mathbb{E}_{q(Z; \phi)\, q(U \mid Z)\, p(f_i \mid U, \hat{x}_i(z_{n_i}); W, \theta)}\!\left[\log \mathbb{E}_{q(\omega_i \mid Z)}\!\left[\frac{\mathrm{PG}(\omega_i \mid 1, 0)}{q(\omega_i \mid Z)} \cdot \ell(y_i \mid f_i, \omega_i)\right]\right] \nonumber \\
    &\quad - \mathbb{E}_{q(Z; \phi)}\!\left[D_{\mathrm{KL}}\!\left[q(U \mid Z) \,\|\, p(U; W, \theta)\right]\right] - D_{\mathrm{KL}}\!\left[q(Z; \phi) \,\|\, p(Z)\right] \\
    &\geq \sum_{i=1} \mathbb{E}_{q(Z; \phi)\, q(U \mid Z)\, p(f_i \mid U, \hat{x}_i(z_{n_i}); W, \theta)\, q(\omega_i \mid Z)}\!\left[\log\frac{\mathrm{PG}(\omega_i \mid 1, 0)}{q(\omega_i \mid Z)} + \log \ell(y_i \mid f_i, \omega_i)\right] \nonumber \\
    &\quad - \mathbb{E}_{q(Z; \phi)}\!\left[D_{\mathrm{KL}}\!\left[q(U \mid Z) \,\|\, p(U; W, \theta)\right]\right] - D_{\mathrm{KL}}\!\left[q(Z; \phi) \,\|\, p(Z)\right] \\
    &= \sum_{i=1} \mathbb{E}_{q(Z; \phi)\, q(U \mid Z)\, p(f_i \mid U, \hat{x}_i(z_{n_i}); W, \theta)\, q(\omega_i \mid Z)}\!\left[\log \ell(y_i \mid f_i, \omega_i)\right] \nonumber \\
    &\quad - \sum_{i=1} \mathbb{E}_{q(Z; \phi)}\!\left[D_{\mathrm{KL}}\!\left[q(\omega_i \mid Z) \,\|\, \mathrm{PG}(\omega_i \mid 1, 0)\right]\right] \nonumber \\
    &\quad - \mathbb{E}_{q(Z; \phi)}\!\left[D_{\mathrm{KL}}\!\left[q(U \mid Z) \,\|\, p(U; W, \theta)\right]\right] - D_{\mathrm{KL}}\!\left[q(Z; \phi) \,\|\, p(Z)\right] \\
    &=: \mathcal{L}_{\mathrm{PG-ELBO}}
\end{align}
This provides us with a variational lower bound we can optimize via VEM.

\textbf{Closed-Form P\'olya-Gamma ELBO.}
\cref{eq:pg_elbo} defines the variational lower bound with terms expressed as expectations. We now derive analytical expressions for these terms, with the exception of expectations with respect to $q(Z; \phi)$, which lack a closed-form solution and are estimated via Monte Carlo sampling. To this end, we specify a Gaussian inducing factor $q(U \mid Z) = \mathcal{N}(\mu(Z), \Sigma(Z))$ and a P\'olya-Gamma auxiliary factor $q(\omega_i \mid Z) = \operatorname{PG}(1, c_i(Z))$ parameterized by tilt parameters $c_i(Z) \ge 0$. As demonstrated later, the optimal coordinate-ascent updates take precisely these parametric forms; hence, this choice incurs no loss of generality at the fixed point.
We do this term by term:
\begin{itemize}
\item \textbf{The P\'olya-Gamma Mean.} 
From~\citet{polson2013bayesian}, we know the mean of a P\'olya-Gamma distribution is
\begin{align}
    \mathbb{E}_{\mathrm{PG}(b, c)}\!\left[\omega\right]
    &= \frac{b}{2c}\tanh\!\left(\frac{c}{2}\right),
    \label{eq:pg-mean}
\end{align}
a fact we use in the derivation of the reconstruction term and the coordinate-ascent updates below.
Since here we only care about the case in which $b = 1$, we define:
\begin{align}
    \bar\omega(c) := \mathbb{E}_{\mathrm{PG}(1, c)}\!\left[\omega\right] = \frac{1}{2c}\tanh\!\left(\frac{c}{2}\right). \label{eq:pg-mean-b1}
\end{align}

\item \textbf{The Reconstruction Term Given $Z$.} Taking the expectation of \cref{eq:log_ell} over $q(\omega_i \mid Z)\, q(U \mid Z)\, p(f_i \mid U, \hat{x}_i(z_{n_i}); W, \theta)$, and writing $m_i := m_\theta(x_i)$ for the mean-function value at row $i$, we obtain the augmented log-likelihood:
\begin{align}
    &\mathbb{E}_{q(\omega_i \mid Z)\, q(U \mid Z)\, p(f_i \mid U, \hat{x}_i(z_{n_i}); W, \theta)}\!\left[\log \ell(y_i \mid f_i, \omega_i)\right] \nonumber \\
    &\qquad = -\log 2
    + \tfrac{2y_i - 1}{2}\, \mathbb{E}_{q(f_i \mid Z, X; W, \theta)}[f_i]
    - \frac{\bar\omega\big(c_i(Z)\big)}{2}\,\mathbb{E}_{q(f_i \mid Z, X; W, \theta)}\!\left[f_i^2\right] \\
    &\qquad = -\log 2
    + \tfrac{2y_i - 1}{2}\, \mathbb{E}_{q(f_i \mid Z, X; W, \theta)}[f_i]
    - \frac{\bar\omega\big(c_i(Z)\big)}{2} \left( \mathbb{V}_{q(f_i \mid Z, X; W, \theta)}\!\left[f_i\right] + \mathbb{E}_{q(f_i \mid Z, X; W, \theta)}\!\left[f_i\right]^2 \right)
    \label{eq:pg_datafit_moments}
\end{align}
where the two moments of $f_i$ are taken under the Gaussian marginal obtained by integrating $U$ out of $\Circled{C.2}$---the sparse-variational marginal of \citet{hensman2015scalable}, $q(f_i \mid Z, X; W, \theta) = \int q(U \mid Z)\, p(f_i \mid U, \hat{x}_i(z_{n_i}); W, \theta)\, \mathrm{d}U$, in which:
\begin{align}
\begin{split}
    \mathbb{E}_{q(f_i \mid Z, X; W, \theta)}[f_i] &= m_i + \psi_i(z_{n_i})\, \mu(Z), \\
    \mathbb{V}_{q(f_i \mid Z, X; W, \theta)}[f_i] &= K_\theta(\hat{x}_i(z_{n_i}), \hat{x}_i(z_{n_i})) + \psi_i(z_{n_i})\, \big(\Sigma(Z) - K_\theta(W, W)\big)\, \psi_i(z_{n_i})^\intercal .    
\end{split} \label{eq:pg_f_moments}
\end{align}

\item \textbf{The Auxiliary KL Term Given $Z$.} From~\citet{polson2013bayesian}, we have:
\begin{align}
    \mathrm{PG}(\omega_i \mid 1, c_i(Z)) &= \cosh\!\left(\frac{c_i(Z)}{2}\right) \exp\!\left(-\frac{\omega_i c_i(Z)^2}{2}\right) \mathrm{PG}(\omega_i \mid 1, 0). \label{eq:pg_tilt}
\end{align}
Therefore,
\begin{align}
    D_{\mathrm{KL}}\!\left[\mathrm{PG}(1, c_i(Z)) \,\|\, \mathrm{PG}(1, 0)\right]
    &= \mathbb{E}_{\mathrm{PG}(1, c_i(Z))}\!\left[\log \cosh\!\left(\frac{c_i(Z)}{2}\right) - \frac{\omega_i c_i(Z)^2}{2}\right] \nonumber \\
    &= \log \cosh\!\left(\frac{c_i(Z)}{2}\right) - \frac{c_i(Z)^2}{2}\,\bar\omega\big(c_i(Z)\big) \nonumber \\
    &= \log \cosh\!\left(\frac{c_i(Z)}{2}\right) - \frac{c_i(Z)}{4}\tanh\!\left(\frac{c_i(Z)}{2}\right) . \label{eq:pg_kl_omega}
\end{align}

\item \textbf{The Remaining KL Terms.} 
The inducing-point KL term, $D_{\mathrm{KL}}\!\left[q(U \mid Z) \,\|\, p(U; W, \theta)\right]$, is a KL between two Gaussians and admits an analytic form.
In this work we use a mean-field diagonal Gaussian variational family over the latent codes, giving the standard KL term, $D_{\mathrm{KL}}[q(Z;\phi) \,\|\, p(Z)] = \sum_{n=1}^N D_{\mathrm{KL}}[q(z_n; \phi) \,\|\, \mathcal{N}(0, \mathbb{I}_{D_z})]$.
\end{itemize}
Substituting the three terms above into \cref{eq:pg_elbo}, we get:
\begin{align}
    \mathcal{L}_{\mathrm{PG-ELBO}}
    &= \mathbb{E}_{q(Z; \phi)}\Bigg[ \sum_{i=1}^{I} \bigg(
        -\log 2 - \log \cosh\!\left(\frac{c_i(Z)}{2}\right) + \frac{c_i(Z)}{4}\tanh\!\left(\frac{c_i(Z)}{2}\right)
        + \tfrac{2y_i - 1}{2} \big(m_i + \psi_i(z_{n_i})\, \mu(Z)\big) \nonumber \\
    &\qquad\qquad - \frac{\bar\omega\big(c_i(Z)\big)}{2}\Big(
        \big(m_i + \psi_i(z_{n_i})\, \mu(Z)\big)^2
        + K_\theta(\hat{x}_i(z_{n_i}), \hat{x}_i(z_{n_i}))
        + \psi_i(z_{n_i})\, \big(\Sigma(Z) - K_\theta(W, W)\big)\, \psi_i(z_{n_i})^\intercal
        \Big) \bigg) \nonumber \\
    &\qquad\qquad - D_{\mathrm{KL}}\!\left[q(U \mid Z) \,\|\, p(U; W, \theta)\right] \Bigg]
    - D_{\mathrm{KL}}\!\left[q(Z; \phi) \,\|\, p(Z)\right] . \label{eq:pg_elbo_closed}
\end{align}
The per-observation sum is what permits mini-batching: for a batch $\mathcal{B}$ of $|\mathcal{B}|$ rows drawn uniformly, rescaling that inner sum by $I / |\mathcal{B}|$ is unbiased, while the two KL terms are not rescaled.

\subsection{Coordinate-Ascent Updates on the Sparse Variational P\'olya-Gamma Lower Bound}

\textbf{Update for $q(\omega_i \mid Z)$.}
We adopt the approach of~\citet{wenzel2019efficient} to derive the $q(\omega_i \mid Z)$ update.
The terms in $\mathcal{L}_{\mathrm{PG-ELBO}}$ (\cref{eq:pg_elbo}) that depend on $q(\omega_i \mid Z)$ are:
\begin{align}
    \mathcal{L}_{q(\omega_i \mid Z)}
    &= \mathbb{E}_{q(Z; \phi)}\!\left[
        \mathbb{E}_{q(U \mid Z)\, p(f_i \mid U, \hat{x}_i(z_{n_i}); W, \theta)\, q(\omega_i \mid Z)}\!\left[\log \ell(y_i \mid f_i, \omega_i)\right] \right. \left. - D_{\mathrm{KL}}\!\left[q(\omega_i \mid Z) \,\|\, \mathrm{PG}(\omega_i \mid 1, 0)\right]
    \right]
\end{align}
Taking the functional derivative with respect to $q(\omega_i \mid Z)$ for fixed $Z$, and writing out the expectation as an integral, we get:
\begin{align}
    &\frac{\delta \mathcal{L}_{\mathrm{PG-ELBO}}}{\delta q(\omega_i \mid Z)} \nonumber \\
    &= q(Z; \phi) \cdot \frac{\delta}{\delta q(\omega_i \mid Z)} \int q(\omega_i \mid Z) \Big[
        \mathbb{E}_{q(U \mid Z)\, p(f_i \mid U, \hat{x}_i(z_{n_i}); W, \theta)}\!\left[\log \ell(y_i \mid f_i, \omega_i)\right] + \log \mathrm{PG}(\omega_i \mid 1, 0) 
        - \log q(\omega_i \mid Z)
    \Big] \mathrm{d}\omega_i \\
    &= q(Z; \phi) \cdot \Big[
        \mathbb{E}_{q(U \mid Z)\, p(f_i \mid U, \hat{x}_i(z_{n_i}); W, \theta)}\!\left[\log \ell(y_i \mid f_i, \omega_i)\right] + \log \mathrm{PG}(\omega_i \mid 1, 0) - \log q(\omega_i \mid Z) - 1
    \Big]
\end{align}
Since $q(Z; \phi) > 0$, setting $\frac{\delta \mathcal{L}_{\mathrm{PG-ELBO}}}{\delta q(\omega_i \mid Z)} = 0$ and solving for $\log q^*(\omega_i \mid Z)$, we get:
\begin{align}
    \log q^*(\omega_i \mid Z)
    &= \mathbb{E}_{q(U \mid Z)\, p(f_i \mid U, \hat{x}_i(z_{n_i}); W, \theta)}\!\left[\log \ell(y_i \mid f_i, \omega_i)\right] + \log \mathrm{PG}(\omega_i \mid 1, 0) + \mathrm{const} ,
\end{align}
where the $-1$ is absorbed into the normalizing constant. 
Substituting in $\log \ell(y_i \mid f_i, \omega_i)$ from \cref{eq:log_ell} and taking the expectation term by term, we get:
\begin{align}
    &\mathbb{E}_{q(U \mid Z)\, p(f_i \mid U, \hat{x}_i(z_{n_i}); W, \theta)}\!\left[\log \ell(y_i \mid f_i, \omega_i)\right] \nonumber \\
    &= -\log 2 
    + \underbrace{\frac{2y_i - 1}{2}\mathbb{E}_{q(U \mid Z)\, p(f_i \mid U, \hat{x}_i(z_{n_i}); W, \theta)}\!\left[f_i\right]}_{\text{const w.r.t. }\omega_i} 
    - \frac{\omega_i}{2}\underbrace{\mathbb{E}_{q(U \mid Z)\, p(f_i \mid U, \hat{x}_i(z_{n_i}); W, \theta)}\!\left[f_i^2\right]}_{=:\,c_i(Z)^2} .
\end{align}
All terms except $-\frac{c_i(Z)^2}{2}\omega_i$ are constant with respect to $\omega_i$, so:
\begin{align}
    \log q^*(\omega_i \mid Z)
    &= -\frac{c_i(Z)^2}{2}\omega_i + \log \mathrm{PG}(\omega_i \mid 1, 0) + \mathrm{const} .
\end{align}
This is exactly the exponential tilt of \cref{eq:pg_tilt}, so
\begin{align}
    q^*(\omega_i \mid Z) = \mathrm{PG}\left(\omega_i \mid 1, c_i(Z)\right).
\end{align}
Therefore, by \cref{eq:pg_f_moments},
\begin{align}
    c_i(Z) &= \sqrt{\mathbb{E}_{q(f_i \mid Z, X; W, \theta)}\!\left[f_i^2\right]} \\
    &= \sqrt{\mathbb{V}_{q(f_i \mid Z, X; W, \theta)}\!\left[f_i\right] + \mathbb{E}_{q(f_i \mid Z, X; W, \theta)}\!\left[f_i\right]^2} \\
    &= \sqrt{\big(m_i + \psi_i(z_{n_i})\, \mu(Z)\big)^2 + K_\theta(\hat{x}_i(z_{n_i}), \hat{x}_i(z_{n_i})) + \psi_i(z_{n_i})\, \big(\Sigma(Z) - K_\theta(W, W)\big)\, \psi_i(z_{n_i})^\intercal} . \label{eq:pg_ci}
\end{align}

\textbf{Update for $q(U \mid Z)$.}
We adopt the approach of~\citet{wenzel2019efficient} to derive the $q(U \mid Z)$ update.
Taking the functional derivative of $\mathcal{L}_{\mathrm{PG-ELBO}}$ with respect to $q(U \mid Z)$ for fixed $Z$, and writing out the expectation as an integral, we get:
\begin{align}
    &\frac{\delta \mathcal{L}_{\mathrm{PG-ELBO}}}{\delta q(U \mid Z)} \nonumber \\
    &= q(Z; \phi) \cdot \frac{\delta}{\delta q(U \mid Z)} \int q(U \mid Z) \Big[
        \sum_{i=1} \mathbb{E}_{q(\omega_i \mid Z)\, p(f_i \mid U, \hat{x}_i(z_{n_i}); W, \theta)}\!\left[\log \ell(y_i \mid f_i, \omega_i)\right] + \log p(U; W, \theta)
        - \log q(U \mid Z)
    \Big] \mathrm{d}U \\
    &= q(Z; \phi) \cdot \Big[
        \sum_{i=1} \mathbb{E}_{q(\omega_i \mid Z)\, p(f_i \mid U, \hat{x}_i(z_{n_i}); W, \theta)}\!\left[\log \ell(y_i \mid f_i, \omega_i)\right] + \log p(U; W, \theta) - \log q(U \mid Z) - 1 
    \Big] .
\end{align}
Since $q(Z; \phi) > 0$, setting $\frac{\delta \mathcal{L}_{\mathrm{PG-ELBO}}}{\delta q(U \mid Z)} = 0$ and solving for $\log q^*(U \mid Z)$:
\begin{align}
    \log q^*(U \mid Z)
    &= \sum_{i=1} \mathbb{E}_{q(\omega_i \mid Z)\, p(f_i \mid U, \hat{x}_i(z_{n_i}); W, \theta)}\!\left[\log \ell(y_i \mid f_i, \omega_i)\right] + \log p(U; W, \theta) + \mathrm{const}, \label{eq:opt-qu-given-z}
\end{align}
where the $-1$ is absorbed into the normalizing constant.
Substituting in $\log \ell(y_i \mid f_i, \omega_i)$ from \cref{eq:log_ell} and taking the expectation over $p(f_i \mid U, \hat{x}_i(z_{n_i}); W, \theta)$ term by term, we get:
\begin{align}
    &\mathbb{E}_{p(f_i \mid U, \hat{x}_i(z_{n_i}); W, \theta)}\!\left[\log \ell(y_i \mid f_i, \omega_i)\right] \nonumber \\
    &= -\log 2
    + \frac{2y_i - 1}{2}\underbrace{\mathbb{E}_{p(f_i \mid U, \hat{x}_i(z_{n_i}); W, \theta)}\!\left[f_i\right]}_{=\,m_i + \psi_i(z_{n_i}) U}
    - \frac{\omega_i}{2}\underbrace{\mathbb{E}_{p(f_i \mid U, \hat{x}_i(z_{n_i}); W, \theta)}\!\left[f_i^2\right]}_{=\,K_\theta(\hat{x}_i(z_{n_i}), \hat{x}_i(z_{n_i})) - \psi_i(z_{n_i})\, K_\theta(W, W)\, \psi_i(z_{n_i})^\intercal + (m_i + \psi_i(z_{n_i}) U)^2} ,
\end{align}
where $m_i := m_\theta(x_i)$ (so $f_i = m_i + \psi_i(z_{n_i}) U$ conditional on $U$), and we used $\mathbb{E}_{p(f_i \mid U, \hat{x}_i(z_{n_i}); W, \theta)}[f_i] = m_i + \psi_i(z_{n_i}) U$ and:
\begin{align}
    \mathbb{E}_{p(f_i \mid U, \hat{x}_i(z_{n_i}); W, \theta)}[f_i^2]
    &= \mathbb{V}_{p(f_i \mid U, \hat{x}_i(z_{n_i}); W, \theta)}[f_i] + \left(\mathbb{E}_{p(f_i \mid U, \hat{x}_i(z_{n_i}); W, \theta)}[f_i]\right)^2 \nonumber \\
    &= K_\theta(\hat{x}_i(z_{n_i}), \hat{x}_i(z_{n_i})) - \psi_i(z_{n_i})\, K_\theta(W, W)\, \psi_i(z_{n_i})^\intercal + \left(m_i + \psi_i(z_{n_i}) U\right)^2
\end{align}
with $\psi_i(z_{n_i})$ as in \cref{apx:sv}. 
Now, taking the expectation over $q(\omega_i \mid Z)$, whose mean is $\bar\omega\big(c_i(Z)\big)$, we get:
\begin{align}
    &\mathbb{E}_{q(\omega_i \mid Z)\, p(f_i \mid U, \hat{x}_i(z_{n_i}); W, \theta)}\!\left[\log \ell(y_i \mid f_i, \omega_i)\right] \nonumber \\
    &= -\log 2
    + \frac{2y_i - 1}{2}\big(m_i + \psi_i(z_{n_i}) U\big)
    - \frac{\bar\omega\big(c_i(Z)\big)}{2}\Big(\big(m_i + \psi_i(z_{n_i}) U\big)^2 + K_\theta(\hat{x}_i(z_{n_i}), \hat{x}_i(z_{n_i})) - \psi_i(z_{n_i})\, K_\theta(W, W)\, \psi_i(z_{n_i})^\intercal\Big)
\end{align}
using the second moment from above.
Summing over all observations and dropping all terms constant with respect to $U$ (expanding the square, the mean contributes the linear term $-\bar\omega\big(c_i(Z)\big) m_i\,\psi_i U$), we get: 
\begin{align}
    &\sum_{i=1} \mathbb{E}_{q(\omega_i \mid Z)\, p(f_i \mid U, \hat{x}_i(z_{n_i}); W, \theta)}\!\left[\log \ell(y_i \mid f_i, \omega_i)\right] \nonumber \\
    &\qquad= \sum_{i=1}\!\left(\frac{2y_i - 1 - 2\bar\omega\big(c_i(Z)\big) m_i}{2}\,\psi_i(z_{n_i}) U - \frac{\bar\omega\big(c_i(Z)\big)}{2}\left(\psi_i(z_{n_i}) U\right)^2\right) + \mathrm{const} \\
    &\qquad= \frac{1}{2}\big(2Y - 1 - 2\bar{\Omega}(Z)\, m\big)^\intercal \Psi(Z) U
    - \frac{1}{2} U^\intercal \Psi(Z)^\intercal \bar{\Omega}(Z) \Psi(Z) U + \mathrm{const} ,
\end{align}
where $\Psi(Z)$ is the matrix with rows $\psi_i(z_{n_i})$, $\bar{\Omega}(Z) = \mathrm{diag}\big(\bar\omega(c_1(Z)), \ldots, \bar\omega(c_I(Z))\big)$ (both depending on $Z$), and $m = m_\theta(X)$ is the vector of mean-function values.
We add $\log p(U; W, \theta)$ from $\Circled{C.1}$, using $p(U; W, \theta) = \mathcal{N}(U \mid 0, K_\theta(W, W))$:
\begin{align}
    \log p(U; W, \theta)
    &= -\frac{1}{2} U^\intercal K_\theta(W, W)^{-1} U + \mathrm{const}.
\end{align}
We combine all terms:
\begin{align}
    \log q^*(U \mid Z)
    &= \frac{1}{2}(2Y - 1 - 2\bar{\Omega}(Z) m)^\intercal \Psi(Z) U - \frac{1}{2} U^\intercal \Psi(Z)^\intercal \bar{\Omega}(Z) \Psi(Z) U - \frac{1}{2} U^\intercal K_\theta(W, W)^{-1} U + \mathrm{const} \\
    &= \frac{1}{2}(2Y - 1 - 2\bar{\Omega}(Z) m)^\intercal \Psi(Z) U - \frac{1}{2} U^\intercal \underbrace{\left(K_\theta(W, W)^{-1} + \Psi(Z)^\intercal \bar{\Omega}(Z) \Psi(Z)\right)}_{=:\,\Sigma(Z)^{-1}} U + \mathrm{const} \\
    &= U^\intercal \underbrace{\frac{1}{2}\Psi(Z)^\intercal(2Y - 1 - 2\bar{\Omega}(Z) m)}_{=:\,\Sigma(Z)^{-1}\mu(Z)}
    - \frac{1}{2} U^\intercal \Sigma(Z)^{-1} U + \mathrm{const} \\
    &= U^\intercal \Sigma(Z)^{-1} \mu(Z) - \frac{1}{2} U^\intercal \Sigma(Z)^{-1} U + \mathrm{const} .
\end{align}
We now complete the square by expanding $-\frac{1}{2}(U - \mu(Z))^\intercal \Sigma(Z)^{-1}(U - \mu(Z))$:
\begin{align}
    &-\frac{1}{2}(U - \mu(Z))^\intercal \Sigma(Z)^{-1}(U - \mu(Z)) 
    = -\frac{1}{2} U^\intercal \Sigma(Z)^{-1} U + U^\intercal \Sigma(Z)^{-1} \mu(Z) - \frac{1}{2}\mu(Z)^\intercal \Sigma(Z)^{-1} \mu(Z) .
\end{align}
Since $\frac{1}{2}\mu(Z)^\intercal \Sigma(Z)^{-1} \mu(Z)$ does not depend on $U$, it is absorbed into the constant:
\begin{align}
    \log q^*(U \mid Z)
    &= -\frac{1}{2}(U - \mu(Z))^\intercal \Sigma(Z)^{-1}(U - \mu(Z)) + \mathrm{const} . \label{eq:pg_qstar}
\end{align}
This is the log of a Gaussian, so $q^*(U \mid Z) = \mathcal{N}(\mu(Z), \Sigma(Z))$, where $\Sigma(Z)$ and $\mu(Z)$ are obtained by inverting the definitions above:
\begin{align}
    \Sigma(Z)^{-1} &= K_\theta(W, W)^{-1} + \Psi(Z)^\intercal \bar{\Omega}(Z) \Psi(Z) \label{eq:pg_sigma} \\
    &\Rightarrow \quad
    \Sigma(Z) = \left(K_\theta(W, W)^{-1} + \Psi(Z)^\intercal \bar{\Omega}(Z) \Psi(Z)\right)^{-1} \\
    \Sigma(Z)^{-1}\mu(Z) &= \frac{1}{2}\Psi(Z)^\intercal(2Y - 1 - 2\bar{\Omega}(Z) m) \nonumber \\
    &\Rightarrow \quad
    \mu(Z) = \frac{1}{2}\,\Sigma(Z)\,\Psi(Z)^\intercal (2Y - 1 - 2\bar{\Omega}(Z) m)
\end{align}
so:
\begin{align}
    q^*(U \mid Z) = \mathcal{N}(\mu(Z), \Sigma(Z))
\end{align}

\subsection{Variational Expectation-Maximization (VEM) Inference Algorithm}

\begin{algorithm}[t]
\caption{A Single Iteration of the VEM Algorithm for PG-LSGP.}
\label{alg:pg_cavi}
\begin{algorithmic}[1]
\Require parameters $\{\theta, W, \phi\}$; rows $(X, Y)$; inner sweeps $T$; Monte Carlo samples $J$.
\For{$j = 1$ \textbf{to} $J$} \Comment{Monte Carlo samples of the codes}
  \State draw $Z^{(j)} \sim q(Z; \phi)$ (reparameterized)
  \State $c^{(0)}_i \gets 1$ for every row $i$ \Comment{initialize the P\'olya-Gamma tilt parameters}
  \Statex \quad \textbf{E-step} (closed-form CAVI over the conjugate factors; every iterate below is at $Z^{(j)}$, so the argument is left implicit):
  \For{$t = 1$ \textbf{to} $T$} \Comment{alternate the two closed-form conjugate updates}
    \For{each row $i$}
      \State $\bar\omega^{(t)}_i \gets \bar\omega\big(c^{(t-1)}_i\big)$ \Comment{PG mean (\cref{eq:pg-mean-b1})}
    \EndFor
    \State $\big(\Sigma^{(t)}\big)^{-1} \gets K_\theta(W,W)^{-1} + \Psi(Z^{(j)})^\intercal \bar\Omega^{(t)}\, \Psi(Z^{(j)})$ \Comment{closed-form $q^*(U\mid Z^{(j)})$ (\cref{eq:pg_sigma})}
    \State $\mu^{(t)} \gets \tfrac12\,\Sigma^{(t)}\,\Psi(Z^{(j)})^\intercal\big(2Y-1 - 2\bar\Omega^{(t)}\, m\big)$
    \For{each row $i$}
      \State $c^{(t)}_i \gets \sqrt{\mathbb{E}_{q(f_i \mid Z^{(j)}, X; W, \theta)}[f_i^2]}$ at $\big(\mu^{(t)}, \Sigma^{(t)}\big)$ \Comment{refresh tilt parameters (\cref{eq:pg_ci})}
    \EndFor
  \EndFor
  \State Freeze $c \gets c^{(T)}$ via stop-gradient \Comment{envelope theorem: $\partial\mathcal{L}_{\mathrm{PG-ELBO}}/\partial c = 0$ at the fixed point}
  \State Recompute $\bar\Omega(Z^{(j)}), \Sigma(Z^{(j)}), \mu(Z^{(j)})$ in closed form at the frozen $c$ \Comment{differentiated through}
  \State $\mathcal{L}^{(j)} \gets$ the bracketed integrand of \cref{eq:pg_elbo_closed} at $\big(c, \mu(Z^{(j)}), \Sigma(Z^{(j)})\big)$ 
\EndFor
\Statex \textbf{M-step} (gradient step on the non-conjugate parameters):
\State $\mathcal{L}_{\mathrm{PG-ELBO}} \gets \tfrac{1}{J}\sum_{j=1}^{J} \mathcal{L}^{(j)} - D_{\mathrm{KL}}[q(Z; \phi) \,\|\, p(Z)]$ \Comment{Monte Carlo estimate of \cref{eq:pg_elbo_closed}}
\State $\{\theta, W, \phi\} \gets \textsc{AdamStep}(\{\theta, W, \phi\},\, \nabla_{\{\theta,W,\phi\}} \mathcal{L}_{\mathrm{PG-ELBO}})$
\end{algorithmic}
\end{algorithm}

We perform inference via a Variational Expectation-Maximization (VEM) algorithm (\cref{alg:pg_cavi}) that alternates between a conjugate E-step and a non-conjugate M-step. In the \textbf{E-step}, we execute Coordinate-Ascent Variational Inference (CAVI) over the conjugate factors, which exhibits fast, monotonic convergence to its fixed point. Crucially, because the variational parameters $c$ reach the CAVI optimum where $\partial \mathcal{L}_{\mathrm{PG-ELBO}} / \partial c = 0$, the envelope theorem guarantees that stopping gradients through the inner loop yields exact total derivatives for the non-conjugate parameters $\{\theta, W, \phi\}$ (kernel hyperparameters, inducing inputs, and remaining variational parameters). This stop-gradient eliminates the need to store intermediate states across $T$ sweeps during reverse-mode automatic differentiation---drastically reducing memory overhead for large inducing sets $M$---without altering the forward value. In the \textbf{M-step}, we take a plain Euclidean gradient step on $\{\theta, W, \phi\}$, differentiating through the closed-form distribution $q^*(U \mid Z)$, which is recomputed at the current parameters and thus envelope-exact as well. Unlike Conjugate-Computation Variational Inference \citep{khan2017conjugate}, which recasts non-conjugate updates via mirror descent in mean-parameter space, our M-step uses standard Euclidean gradients and naturally admits mini-batching over observation rows.

\textbf{Numerical Stability of the P\'olya-Gamma Mean.}
\Cref{eq:pg-mean-b1} has a removable singularity at $c = 0$, where the numerator and denominator vanish together.
Expanding the numerator resolves it: 
\begin{align}
    \bar\omega(c)
    &= \frac{1}{2c}\left[u - \frac{u^3}{3} + O(u^5)\right]\bigg|_{u = c/2}
    && \big(\tanh u = u - \tfrac{u^3}{3} + O(u^5)\big) \nonumber \\
    &= \frac{1}{2c}\left[\frac{c}{2} - \frac{c^3}{24} + O(c^5)\right]
    && \big(\text{substituting } u = c/2\big) \nonumber \\
    &= \frac{1}{4} - \frac{c^2}{48} + O(c^4) . \label{eq:pg_omega_taylor}
\end{align}
The leading term gives the limit, $\lim_{c \to 0}\bar\omega(c) = \tfrac14$, so $\bar\omega$ extends continuously through $c = 0$.
In floating point the two cases differ: evaluated directly, \cref{eq:pg-mean-b1} is accurate for every $c \neq 0$, but at $c = 0$ exactly it returns \texttt{NaN} rather than $\tfrac14$, and its derivative loses precision to cancellation as $c \to 0$.
We therefore pick a small threshold ($\varepsilon = 10^{-6}$) and evaluate \cref{eq:pg_omega_taylor} below it and \cref{eq:pg-mean-b1} above it:
\begin{align}
    \bar\omega(c) \approx
    \mathbb{I}\!\left[\,|c| < \varepsilon\,\right] \left(\frac{1}{4} - \frac{c^2}{48}\right)
    \;+\;
    \mathbb{I}\!\left[\,|c| \ge \varepsilon\,\right] \frac{1}{2|c|}\tanh\!\left(\frac{|c|}{2}\right),
    \qquad \varepsilon = 10^{-6} .
    \label{eq:pg_omega_impl}
\end{align}
This guards the \emph{forward} value; the gradient is not at issue in \cref{alg:pg_cavi}, where $c$ is frozen by stop-gradient before the M-step.

\textbf{Mini-Batching and Stochastic Natural Gradients.}
To scale training to more rows, we can execute \cref{alg:pg_cavi} using mini-batches. For a mini-batch $\mathcal{B} \subseteq \{1, \dots, I\}$, we apply the dataset scaling factor $I / |\mathcal{B}|$ in two places:
\begin{enumerate}
\item \textbf{Reconstruction Term.}
The likelihood reconstruction sum in $\mathcal{L}_{\mathrm{PG-ELBO}}$ is rescaled by $I / |\mathcal{B}|$, leaving the global KL divergence terms unscaled.

\item \textbf{Conjugate Updates.}
Inside the conjugate updates for $q(U \mid Z)$, the posterior moments depend on two full-dataset row sums:
\begin{align}
    \Circled{1} &:= \Psi(Z)^\intercal \bar\Omega(Z) \Psi(Z) = \sum_{i=1}^{I} \bar\omega\big(c_i(Z)\big)\, \psi_i(z_{n_i})^\intercal \psi_i(z_{n_i}), \nonumber \\
    \Circled{2} &:= \tfrac{1}{2}\Psi(Z)^\intercal\big(2Y - 1 - 2\bar\Omega(Z)\, m\big) = \sum_{i=1}^{I} \Big(\tfrac{2y_i - 1}{2} - \bar\omega\big(c_i(Z)\big)\, m_i\Big) \psi_i(z_{n_i})^\intercal, \label{eq:pg_suffstats}
\end{align}
which yield $\Sigma(Z)^{-1} = K_\theta(W, W)^{-1} + \Circled{1}$ and $\mu(Z) = \Sigma(Z) \Circled{2}$. On a mini-batch $\mathcal{B}$, we replace these full-dataset sums with their rescaled, unbiased estimators:
\begin{equation}
    \widehat{\Circled{1}} := \frac{I}{|\mathcal{B}|}\sum_{i \in \mathcal{B}} \bar\omega\big(c_i(Z)\big)\, \psi_i(z_{n_i})^\intercal \psi_i(z_{n_i}), \quad
    \widehat{\Circled{2}} := \frac{I}{|\mathcal{B}|}\sum_{i \in \mathcal{B}} \Big(\tfrac{2y_i - 1}{2} - \bar\omega\big(c_i(Z)\big)\, m_i\Big) \psi_i(z_{n_i})^\intercal. \label{eq:pg_suffstats_hat}
\end{equation}
\end{enumerate}
Solving the conjugate system in closed form is a stochastic natural-gradient step at unit step size~\citep{wenzel2019efficient}.
Finally, we note that while the estimated statistics $\widehat{\Circled{1}}$ and $\widehat{\Circled{2}}$ are unbiased, the matrix inversion required to compute $\Sigma(Z)$ introduces a small bias of order $O(1 / |\mathcal{B}|)$. As a result, training ELBO values across different batch sizes are not directly comparable, and model selection should be evaluated on held-out log-likelihoods.

\subsection{The Factorized P\'olya-Gamma Bound} \label{apx:pg-factorized}

To empirically demonstrate personalization collapse, we need to derive a factorized bound.
To do this, we replace the code-conditional factor $q(U \mid Z)$ by a free, full-rank multivariate Gaussian that is independent of $Z$:
\begin{align}
    q(U) &= \mathcal{N}(\mu_U,\, \Sigma_U), \qquad \Sigma_U = L_U L_U^\intercal ,
\end{align}
with $\mu_U \in \mathbb{R}^M$ and $L_U$ a lower-triangular Cholesky factor with positive diagonal, both free variational parameters ($\mu_U, L_U \in \phi$), optimized via gradient descent. 
This gives us \cref{alg:pg_fac}.

\begin{algorithm}[t]
\caption{A Single Iteration for the \emph{Factorized} PG-LSGP.}
\label{alg:pg_fac}
\begin{algorithmic}[1]
\Require parameters $\{\theta, W, \phi\}$; rows $(X, Y)$ (or a mini-batch $\mathcal{B}$, with scale $I/|\mathcal{B}|$); code samples $J$
\State $\Sigma_U \gets L_U L_U^\intercal$ \Comment{one $q(U)$ for every draw: it does not read the codes}
\For{$j = 1$ \textbf{to} $J$} \Comment{Monte Carlo samples of the codes}
  \State draw $Z^{(j)} \sim q(Z; \phi)$ (reparameterized)
  \For{each row $i$}
    \State $c_i \gets \sqrt{\mathbb{E}_{q(f_i \mid z_{n_i}, X; W, \theta)}[f_i^2]}$ at $\big(\mu_U, \Sigma_U\big)$ \Comment{\cref{eq:pg_ci} under the substitution: one step, no sweeps}
  \EndFor
  \State Freeze $c$ via stop-gradient \Comment{envelope theorem, as in \cref{alg:pg_cavi}}
  \State $\mathcal{L}^{(j)} \gets$ the row sum of \cref{eq:pg_elbo_closed} at $\big(c, \mu_U, \Sigma_U\big)$
\EndFor
\Statex \textbf{M-step} (a single gradient step; no E-step precedes it):
\State $\mathcal{L}^{\mathrm{fac}}_{\mathrm{PG-ELBO}} \gets \tfrac{1}{J}\sum_{j=1}^{J} \mathcal{L}^{(j)} - D_{\mathrm{KL}}[q(U) \,\|\, p(U; W, \theta)] - D_{\mathrm{KL}}[q(Z;\phi) \,\|\, p(Z)]$ \Comment{neither KL depends on $Z$}
\State $\{\theta, W, \phi\} \gets \textsc{AdamStep}\big(\{\theta, W, \phi\},\, \nabla \mathcal{L}^{\mathrm{fac}}_{\mathrm{PG-ELBO}}\big)$
\end{algorithmic}
\end{algorithm}

\section{THEORY FOR PERSONALIZATION COLLAPSE} \label{apx:bias}

\subsection{Proof That the Coupling Gap Is a Bregman Information} \label{apx:proof-gap}

We express the multivariate Gaussian family in natural parameters as:
\begin{align}
    q_\nu(U) &:= \exp\Big( \langle \nu, T(U) \rangle - \mathcal{A}(\nu) - \tfrac{M}{2}\log 2\pi \Big) ,
\end{align}
where $\nu = (\nu_1, \nu_2)$ denotes the natural parameters, and the sufficient statistic $T(U)$ and log-partition function $\mathcal{A}(\nu)$ \citep{wainwright2008graphical} are defined as:
\begin{align}
    T(U) &:= \Big( U, \, -\tfrac{1}{2} UU^\intercal \Big) , \\
    \mathcal{A}(\nu) &:= \tfrac{1}{2}\nu_1^\intercal \nu_2^{-1}\nu_1 - \tfrac{1}{2}\log|\nu_2| .
\end{align}
Correspondingly, the code-conditional natural parameters $\eta(Z) = \big(\eta_1(Z), \eta_2(Z)\big)$ constructed from \cref{eq:pg_sigma} are given by:
\begin{align}
    \eta_1(Z) &= \Sigma(Z)^{-1}\mu(Z) , \qquad \eta_2(Z) = \Sigma(Z)^{-1} .
\end{align}

\begin{proof}[Proof of \cref{thm:gap}]
We begin by establishing that $q^*(U \mid Z)$ and $\bar q(U) \propto \exp\big(\mathbb{E}_{q(Z; \phi)}\big[\log q^*(U \mid Z)\big]\big)$ are Gaussian (Parts 1 and 2, respectively).
Then we show that the GAP is an expectation of KL divergences (Part 3), which can be expressed as a Bregman information (Part 4).

\textbf{Part 1: $q^*(U \mid Z)$ is Gaussian.}
With all other variational factors held fixed, the coordinate-ascent variational inference (CAVI) optimum for the inducing values satisfies
\begin{align}
    \log \tilde q^*(U \mid Z) = \sum_i \mathbb{E}_{q(\omega_i \mid Z)\, p(f_i \mid U, \hat{x}_i(z_{n_i}); W, \theta)}\!\left[\log \ell(y_i \mid f_i, \omega_i)\right] + \log p(U; W, \theta) ,
    \label{eq:bary_recall}
\end{align}
where $\tilde q^*(U \mid Z)$ is an unnormalized density with normalizer $\mathcal{C}(Z) := \int \tilde q^*(U \mid Z)\,\mathrm{d}U$ and normalized distribution $q^*(U \mid Z) := \tilde q^*(U \mid Z)/\mathcal{C}(Z)$. Completing the square in $U$ (\cref{eq:pg_qstar,eq:pg_sigma}) confirms that $q^*(U \mid Z)$ is Gaussian:
\begin{align}
\begin{split}
    q^*(U \mid Z) &= \mathcal{N}\big(\mu(Z),\, \Sigma(Z)\big), \\
    \Sigma(Z)^{-1} &= K_\theta(W, W)^{-1} + \Psi(Z)^\intercal \bar{\Omega}(Z) \Psi(Z), \\
    \Sigma(Z)^{-1}\mu(Z) &= \tfrac{1}{2}\Psi(Z)^\intercal\big(2Y - 1 - 2\bar{\Omega}(Z) m\big) .
\end{split} \label{eq:bary_qstar_recall}
\end{align}

\textbf{Part 2: $\bar q(U) \propto \exp\big(\mathbb{E}_{q(Z; \phi)}\big[\log q^*(U \mid Z)\big]\big)$ is Gaussian.}
By Part 1, $q^*(U \mid Z) = q_{\eta(Z)}(U)$ is a member of the family above, with natural parameters $\eta(Z) = \big(\Sigma(Z)^{-1}\mu(Z),\, \Sigma(Z)^{-1}\big)$. Since $\log q_\nu(U)$ is \emph{linear} in $\nu$, taking the expectation over $Z$ simply averages the natural parameters:
\begin{align}
    \mathbb{E}_{q(Z; \phi)}\big[\log q_{\eta(Z)}(U)\big]
    = \big\langle \bar\eta,\, T(U) \big\rangle - \mathbb{E}_{q(Z; \phi)}\big[\mathcal{A}(\eta(Z))\big] - \tfrac{M}{2}\log 2\pi ,
    \qquad \bar\eta := \mathbb{E}_{q(Z; \phi)}\big[\eta(Z)\big] .
    \label{eq:bary_eta}
\end{align}
The only $U$-dependence is the inner product $\langle \bar\eta, T(U)\rangle$; every remaining term is a constant. Exponentiating and normalizing therefore gives
\begin{align}
    \bar q(U) = q_{\bar\eta}(U) ,
    \label{eq:bary_square}
\end{align}
the Gaussian at the \emph{expectation of the natural parameters}. Reading $\bar\eta$ back in moment form,
\begin{align}
    \bar\Sigma = \Big(\mathbb{E}_{q(Z; \phi)}\big[\Sigma(Z)^{-1}\big]\Big)^{-1} , \qquad
    \bar\mu = \bar\Sigma\, \mathbb{E}_{q(Z; \phi)}\big[\Sigma(Z)^{-1}\mu(Z)\big] ,
    \label{eq:barycenter}
\end{align}
so it is the \emph{precisions} that average, not the covariances. Normalizability is immediate: $\bar\Sigma^{-1}$ is an expectation of positive-definite precision matrices under a probability measure, hence positive definite.

\textbf{Part 3: GAP is an Expected KL Divergence.} Let $\mathcal{Q} \in \{ q(U \mid Z), q(U) \}$.
As a function of $\mathcal{Q}$ (with all other factors fixed), writing $\mathrm{const}$ for terms independent of $\mathcal{Q}$, we express the optimization of $\mathcal{L}_{\mathrm{PG-ELBO}}$ as:
\begin{align*}
    \max_{\mathcal{Q}} \mathcal{L}_{\mathrm{PG-ELBO}}
    &= \mathrm{const} + \max_{\mathcal{Q}} \bigg[ \sum_i \mathbb{E}_{q(Z; \phi)\,\mathcal{Q}\,p(f_i\mid U)\,q(\omega_i\mid Z)}\big[\log \ell(y_i\mid f_i,\omega_i)\big] + \mathbb{E}_{q(Z; \phi)\,\mathcal{Q}}\big[\log p(U; W, \theta) - \log \mathcal{Q}\big] \bigg] \\
    &= \mathrm{const} + \max_{\mathcal{Q}} \mathbb{E}_{q(Z; \phi)\,\mathcal{Q}}\big[ \log \tilde q^*(U\mid Z) - \log \mathcal{Q} \big]
        \qquad (\text{by Part 1}) \\
    &= \mathrm{const} + \max_{\mathcal{Q}} \bigg[ \mathbb{E}_{q(Z; \phi)}\big[ \log \mathcal{C}(Z) \big] - \mathbb{E}_{q(Z; \phi)}\big[ D_{\mathrm{KL}}[\mathcal{Q}\|q^*(U\mid Z)] \big] \bigg] \\
    &= \mathrm{const} - \min_{\mathcal{Q}} \mathbb{E}_{q(Z; \phi)}\big[ D_{\mathrm{KL}}[\mathcal{Q}\|q^*(U\mid Z)] \big] ,
\end{align*}
where $q^*(U\mid Z) = \tilde q^*(U\mid Z)/\mathcal{C}(Z) = \mathcal{N}(\mu(Z),\Sigma(Z))$ is the CAVI-optimal conditional distribution (\cref{eq:opt-qu-given-z,eq:pg_qstar}), and $\mathbb{E}_{q(Z; \phi)}[\log \mathcal{C}(Z)]$ is absorbed into $\mathrm{const}$.

Evaluating this objective under both $q(U \mid Z)$ and $q(U)$ yields:
\begin{align}
    \max_{q(U\mid Z)}\mathcal{L}_{\mathrm{PG-ELBO}} 
    &= \mathrm{const} - \underbrace{\min_{q(U\mid Z)}\mathbb{E}_{q(Z; \phi)} \left[ D_{\mathrm{KL}}[q(U\mid Z)\|q^*(U\mid Z)] \right]}_{= 0 \text{ by setting } q(U\mid Z) = q^*(U \mid Z)} = \mathrm{const} , \\
    \max_{q(U)}\mathcal{L}_{\mathrm{PG-ELBO}} 
    &= \mathrm{const} - \min_{q(U)}\Big[ D_{\mathrm{KL}}[q(U)\|\bar q(U)] + \mathbb{E}_{q(Z; \phi)} \left[ D_{\mathrm{KL}}[\bar q(U)\|q^*(U\mid Z)] \right] \Big] \nonumber \\
    &= \mathrm{const} - \mathbb{E}_{q(Z; \phi)} \left[ D_{\mathrm{KL}}[\bar q(U)\|q^*(U\mid Z)]\right] - \underbrace{\min_{q(U)} D_{\mathrm{KL}}[q(U)\|\bar q(U)]}_{= 0 \text{ by setting } q(U) = \bar q(U)} \nonumber \\
    &= \mathrm{const} - \mathbb{E}_{q(Z; \phi)}D_{\mathrm{KL}}[\bar q(U)\|q^*(U\mid Z)] ,
\end{align}
where setting $q(U) = \bar q(U)$ is admissible since $\bar q(U)$ is Gaussian (Part 2).

Finally, subtracting the factorized bound from the structured bound cancels the shared $\mathrm{const}$ terms:
\begin{align}
    \mathrm{GAP} = \max_{q(U\mid Z)}\mathcal{L}_{\mathrm{PG-ELBO}} - \max_{q(U)}\mathcal{L}_{\mathrm{PG-ELBO}}
    = \mathbb{E}_{q(Z; \phi)}\big[D_{\mathrm{KL}}[\bar q(U)\|q^*(U\mid Z)]\big],
\end{align}
expressing the coupling gap directly as an expected KL divergence.

\textbf{Part 4: The Expected KL is a Bregman Information.} 
Parts 1 and 2 put both distributions in the family: $q^*(U \mid Z) = q_{\eta(Z)}(U)$ and $\bar q(U) = q_{\bar\eta}(U)$ with $\bar\eta = \mathbb{E}_{q(Z;\phi)}[\eta(Z)]$ (\cref{eq:bary_eta,eq:bary_square}).
Each KL is then the Bregman divergence of $\mathcal{A}$ (\cref{def:bregman}),
\begin{align}
    D_{\mathrm{KL}}\big[\bar q(U)\,\big\|\,q^*(U\mid Z)\big]
    &= D_{\mathrm{KL}}\big[q_{\bar\eta}(U)\,\big\|\,q_{\eta(Z)}(U)\big] \\
    &= \mathbb{E}_{q_{\bar\eta}(U)}\Big[ \log q_{\bar\eta}(U) - \log q_{\eta(Z)}(U) \Big] \\
    &= \mathbb{E}_{q_{\bar\eta}(U)}\Big[ \big\langle \bar\eta - \eta(Z),\, T(U) \big\rangle \Big] - \mathcal{A}(\bar\eta) + \mathcal{A}(\eta(Z))
        && (\tfrac{M}{2}\log 2\pi \text{ cancels}) \\
    &= \big\langle \bar\eta - \eta(Z),\, \nabla\mathcal{A}(\bar\eta) \big\rangle - \mathcal{A}(\bar\eta) + \mathcal{A}(\eta(Z))
        && \big(\mathbb{E}_{q_\nu(U)}[T(U)] = \nabla\mathcal{A}(\nu)\big) \\
    &= B_{\mathcal{A}}\big(\eta(Z),\, \bar\eta\big)
        && (\text{\cref{def:bregman}}) 
\end{align}
and since $\bar\eta = \mathbb{E}_{q(Z;\phi)}[\eta(Z)]$ is the mean, 
\begin{align}
\mathrm{GAP} = \mathbb{E}_{q(Z;\phi)}[B_{\mathcal{A}}(\eta(Z),\bar\eta)] = I_{\mathcal{A}}(\eta(Z)) = \mathbb{E}_{q(Z;\phi)}[\mathcal{A}(\eta(Z))] - \mathcal{A}(\bar\eta),
\end{align}
by \cref{def:breg_info}, which is~\cref{eq:apx_gap}. By \cref{def:bregman} it is $\ge 0$ and vanishes iff $\eta(Z) = \bar\eta$ $q(Z;\phi)$-a.s., i.e.\ iff $\eta(Z)$---equivalently $q^*(U\mid Z) = q_{\eta(Z)}(U)$---is $q(Z;\phi)$-a.s.\ constant.
\end{proof}

\subsection{Proof That the Gap Vanishes Exactly When Similarity Collapses} 
\label{apx:proof-oversim}

\begin{proof}[Proof of \cref{thm:oversim}]
First, we show that the latent codes $Z$ affect the conditional optimum $q^*(U \mid Z)$ exclusively through the cross-covariance matrix $K_\theta(\widehat{X}(Z), W)$ between observations and inducing points.
By the Hadamard factorization of the Similarity Kernel (\cref{def:iskernel}), cross-covariance entries separate across feature and latent channels:
\begin{align}
    \big[K_\theta(\widehat{X}(Z), W)\big]_{im} = K^x_\theta(x_i, w^x_m)\, K^z_\theta(z_{n_i}, w^z_m) . 
    \label{eq:apx_psi_hadamard}
\end{align}
Given a code-independent auxiliary matrix $\bar\Omega^\star$ (\cref{thm:oversim}\textbf{(iii)}) and base mean $m_\theta(X)$, the natural parameters (\cref{eq:pg_sigma,eq:pg_moments}) satisfy:
\begin{align}
    \eta_1(Z) = \tfrac{1}{2}\Psi(Z)^\intercal\big(2Y - 1 - 2\bar\Omega^\star m_\theta(X)\big) , \qquad
    \eta_2(Z) = K_\theta(W,W)^{-1} + \Psi(Z)^\intercal \bar\Omega^\star \Psi(Z) , \label{eq:apx_eta_of_R}
\end{align}
where $\Psi(Z) := K_\theta(\widehat{X}(Z), W)K_\theta(W,W)^{-1}$, and $K_\theta(W,W)^{-1}$ is code-independent and well defined by \cref{thm:oversim}\textbf{(i)}. Thus, $\eta(Z)$ depends on $Z$ solely through $K_\theta(\widehat{X}(Z), W)$.

\textbf{Part 1 ($\Leftarrow$): Collapse implies a zero coupling gap.}
If $S = 0$ and $\sigma_\theta(z) \equiv \sigma$, then $\zeta_\theta \equiv \rho_\theta(0) = 1$, making $K^z_\theta(z, z') \equiv \sigma^2$. The code-dependent term in \cref{eq:apx_psi_hadamard} becomes constant, rendering $K_\theta(\widehat{X}(Z), W)$---and by extension $\eta(Z)$---independent of $Z$. By \cref{thm:gap}, $\mathrm{GAP} = 0$.

\textbf{Part 2 ($\Rightarrow$): A zero coupling gap forces personalization collapse.}
Assume $\mathrm{GAP} = 0$. By \cref{thm:gap}, $\eta(Z)$ is $q(Z;\phi)$-almost surely constant, i.e.\ constant except possibly on a set of $q(Z;\phi)$-probability zero. No such exceptions exist: if $\eta$ took a different value at some $Z$, then by continuity (\cref{def:iskernel}) it would do so on a whole neighborhood of $Z$, and under full support (\cref{thm:oversim}\textbf{(ii)}) every neighborhood has positive probability. Hence, $\eta(Z)$ is pointwise constant everywhere on $\mathbb{R}^{N D_z}$.

\textit{Step 1: Pointwise constancy of live-inducing covariances.}
By \cref{eq:apx_eta_of_R}, constancy of $\eta_2(Z)$ implies that $\Psi(Z)^\intercal \bar\Omega^\star \Psi(Z)$---and consequently the Gram matrix $K_\theta(\widehat{X}(Z), W)^\intercal \bar\Omega^\star K_\theta(\widehat{X}(Z), W)$---is constant in $Z$. With all patient codes set equal ($z_1 = \dots = z_N = z$) for an arbitrary $z \in \mathbb{R}^{D_z}$, the $m$-th diagonal entry expands as:
\begin{align}
    \Big[K_\theta(\widehat{X}(Z), W)^\intercal\, \bar\Omega^\star\, K_\theta(\widehat{X}(Z), W)\Big]_{mm}
    = b_m\, K^z_\theta(z, w^z_m)^2 , \qquad
    b_m := \sum_{i=1}^{I} \bar\omega_i\, K^x_\theta(x_i, w^x_m)^2 .
    \label{eq:apx_eta2_diag}
\end{align}
For any live inducing point $m$, $b_m > 0$ because $\bar\omega_i > 0$ and $K^x_\theta(x_i, w^x_m) \neq 0$ for at least one observation. Thus, $K^z_\theta(z, w^z_m) \ge 0$ is constant across $z \in \mathbb{R}^{D_z}$ for every live $m$.

\textit{Step 2: Joint live-inducing constraints force $S = 0$.}
Since $K^z_\theta(z, w^z_m)$ is constant across $z$, evaluating it at $z = w^z_m$ fixes its value to $K^z_\theta(w^z_m, w^z_m) = \sigma_\theta(w^z_m)^2$ (since $\zeta_\theta(w^z_m, w^z_m) = 1$). Writing out $K^z_\theta(z, w^z_m) = \sigma_\theta(z)\,\sigma_\theta(w^z_m)\,\zeta_\theta(z, w^z_m)$ (\cref{def:iskernel}), equating it to $\sigma_\theta(w^z_m)^2$, and dividing through by $\sigma_\theta(w^z_m) > 0$ yields:
\begin{align}
    \sigma_\theta(z)\,\zeta_\theta(z, w^z_m) = \sigma_\theta(w^z_m) \qquad \forall z \in \mathbb{R}^{D_z}, \; \text{for every live } m .
    \label{eq:apx_live_const}
\end{align}
Fix $d \in \{1, \dots, D_z\}$ and select live points $m \neq m'$ with distinct warped coordinates $\varphi_\theta(w^z_m)_d \neq \varphi_\theta(w^z_{m'})_d$ (\cref{thm:oversim}\textbf{(v)}). Evaluating \cref{eq:apx_live_const} for $m$ at $z = w^z_{m'}$ and for $m'$ at $z = w^z_m$ gives:
\begin{align*}
    \zeta_\theta(w^z_{m'}, w^z_m) = \frac{\sigma_\theta(w^z_m)}{\sigma_\theta(w^z_{m'})} , \qquad
    \zeta_\theta(w^z_m, w^z_{m'}) = \frac{\sigma_\theta(w^z_{m'})}{\sigma_\theta(w^z_m)} .
\end{align*}
By kernel symmetry, $\zeta_\theta(w^z_m, w^z_{m'}) = \zeta_\theta(w^z_{m'}, w^z_m) =: r$. Multiplying these equations yields $r^2 = 1$. Since $r \in [0, 1]$, we have $r = 1$, which requires:
\begin{align*}
    \rho_\theta\Big(S\big(\varphi_\theta(w^z_m) - \varphi_\theta(w^z_{m'})\big)\Big) = 1 .
\end{align*}
Since $\rho_\theta(v) < 1$ for all $v \neq 0$, this implies $S\big(\varphi_\theta(w^z_m) - \varphi_\theta(w^z_{m'})\big) = 0$. Because $S = \mathrm{diag}(S_1, \dots, S_{D_z})$ is diagonal, the $d$-th coordinate of this vector equation reads $S_d \big(\varphi_\theta(w^z_m)_d - \varphi_\theta(w^z_{m'})_d\big) = 0$, where $S_d$ is the importance of latent dimension $d$ (\cref{def:iskernel}); we know that $\varphi_\theta(w^z_m)_d - \varphi_\theta(w^z_{m'})_d$ is non-zero by the choice of $m, m'$; hence $S_d = 0$. Since $d$ was arbitrary, $S = 0$.

\textit{Step 3: Uniform amplitude constancy and complete personalization collapse.}
With $S = 0$, $\zeta_\theta \equiv 1$. Equation~\labelcref{eq:apx_live_const} reduces to $\sigma_\theta(z) = \sigma_\theta(w^z_m)$ for all $z$, proving that $\sigma_\theta$ is constant (\cref{thm:oversim}\textbf{(a)}). Furthermore, $\zeta_\theta \equiv 1$ implies $[A^{\mathrm{trend}}]_{nm} = \zeta_\theta(z_n, z_m) = 1$ everywhere, and constant $\sigma_\theta$ implies $[A^{\mathrm{amp}}]_{nm} = \min(\sigma_\theta(z_n), \sigma_\theta(z_m)) / \max(\sigma_\theta(z_n), \sigma_\theta(z_m)) = 1$ everywhere, so $A^{\mathrm{comb}} = A^{\mathrm{trend}} \cdot A^{\mathrm{amp}} = 1$ as well, establishing total personalization collapse (\cref{thm:oversim}\textbf{(b)}).
\end{proof}

\section{DATA AND EXPERIMENTAL SETUP} \label{apx:data}

We begin with the components of the experimental setup that are common to both the synthetic and real-data experiments.

\textbf{Evaluation Metrics.}
We evaluate model performance along four complementary dimensions, the last of which requires a known ground truth and is therefore reported on synthetic data alone:
\begin{itemize}[nosep, leftmargin=*]
    \item \textbf{Test Average Precision (AP):} Measures performance across precision-recall thresholds.
    \item \textbf{Test Log-Likelihood (LL):} Assesses probabilistic calibration and goodness-of-fit.
    \item \textbf{Mean Off-Diagonal Similarity (MODS):} For an $N \times N$ similarity matrix $A$ (e.g., $A^{\mathrm{trend}}$, $A^{\mathrm{amp}}$, or $A^{\mathrm{comb}}$), define:
    \begin{equation}
        \mathrm{MODS}(A) := \frac{1}{N(N-1)} \sum_{n \neq m} A_{nm} . \label{eq:mods}
    \end{equation}
    MODS quantifies inter-patient information sharing; values near 1 signal personalization collapse (\cref{thm:oversim}), whereas values near 0 reflect fully independent modeling.
    \item \textbf{Similarity Recovery:} The Spearman rank correlation between the off-diagonal entries of a learned similarity matrix and the corresponding ground-truth patient similarities (only applicable to synthetic data). Recovery is scored on $A^{\mathrm{comb}}$ and complements MODS: MODS asks whether the model distinguishes patients at all, recovery whether the distinctions it draws are the right ones. It is undefined for a similarity matrix that is constant off the diagonal, where no ranking exists.
\end{itemize}
We evaluate AP and LL both at the pooled level (aggregating cohort observations to assess global discrimination) and at the per-patient level (measuring within-patient temporal forecasting). Per-patient AP excludes any patient whose evaluation period contains only one class---no SREs, or no non-SRE surveys---since AP is undefined there. To contextualize performance across baseline risk profiles, we stratify results by whether a patient experienced an SRE during training.

\textbf{Software.}
The LSGP is implemented in NumPyro~\citep{bingham2019pyro} on JAX~\citep{jax2018github}, with the
neural components (e.g.~the DeepSet encoder and the kernel warps) written in Equinox~\citep{kidger2021equinox}. The linear
baselines use \texttt{scikit-learn}~\citep{scikit-learn}, the variational and empirical-Bayes
logistic regressions use \texttt{sklearn-bayes}~\citep{sklearnbayes}, and the hierarchical
(mixed-effects) baselines use \texttt{statsmodels}~\citep{seabold2010statsmodels}. All fits run in double
precision, which the Cholesky steps of the sparse bound require.

\subsection{The ``Hemisphere'' Synthetic Dataset} \label{apx:data-synthetic}\label{apx:setup-hemisphere}

Because EMA data lacks ground-truth patient similarities, we construct a synthetic benchmark---the ``Hemisphere'' dataset---to evaluate similarity recovery alongside predictive accuracy. The dataset satisfies three primary design objectives: (1) providing an explicit ground-truth similarity metric to verify recovery, (2) modeling continuous patient heterogeneity by rotating decision boundaries along a manifold, and (3) generating Likert-formatted responses to mirror real EMA survey structures.

\begin{figure}[t]
    \centering
    \includegraphics[width=\linewidth]{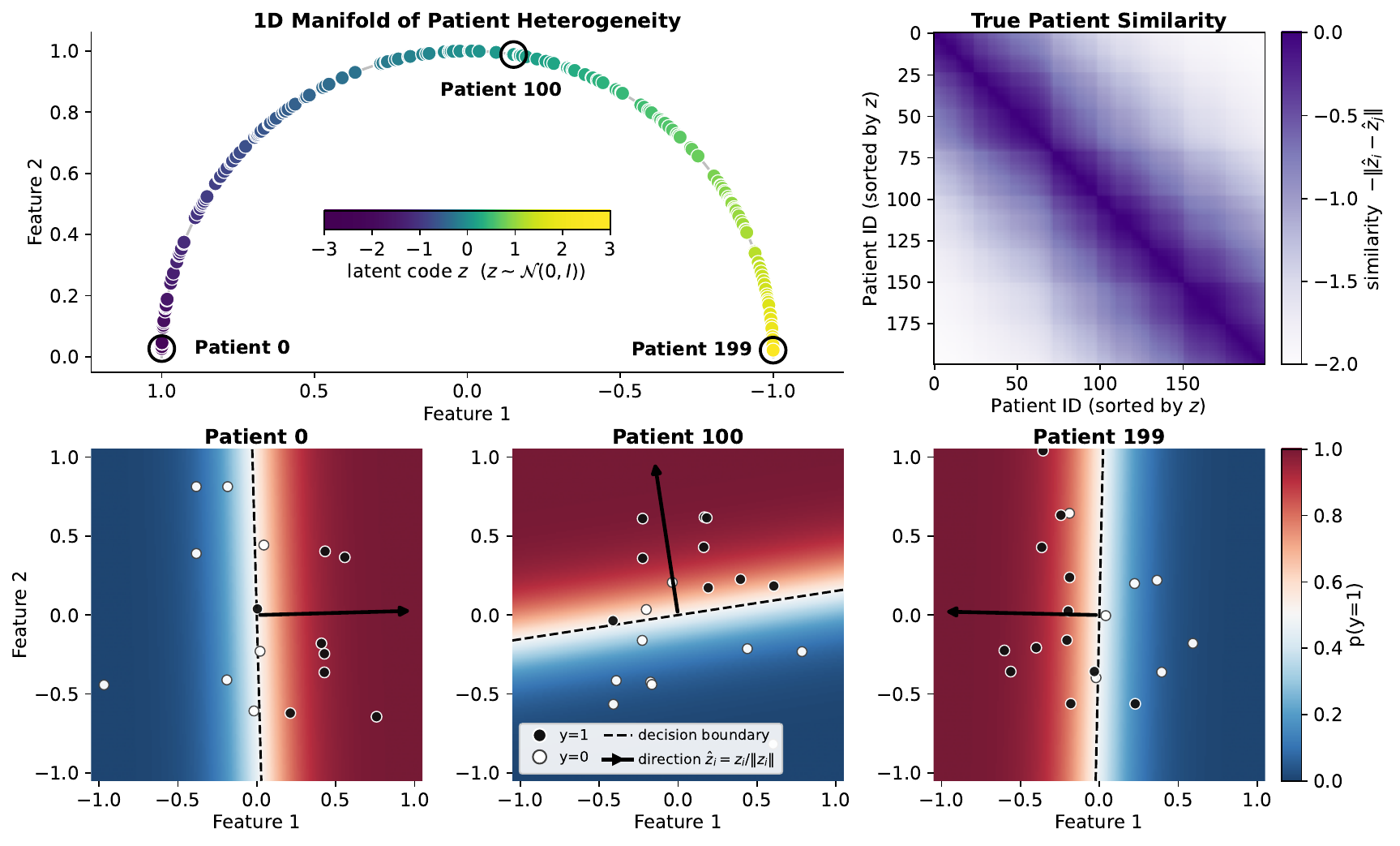}
    \caption{\textbf{Synthetic Hemisphere dataset with ground-truth similarity.} 
    \emph{Top left:} Patient directions $v_n$ on the upper unit semicircle ($D_s = 2$). 
    \emph{Top right:} True similarity matrix $-\lVert v_n - v_{n'} \rVert$ (monotonically equivalent to $v_n^\intercal v_{n'}$), ordered along the manifold. 
    \emph{Bottom row:} Observed feature space for three representative patients, showing patient-specific decision boundaries and survey outcomes.}
    \label{fig:synthetic}
\end{figure}

\textbf{Data-Generating Process.}
Each patient $n \in \{1, \dots, N\}$ is assigned a latent direction $v_n$ on the upper hemisphere of a unit sphere $\mathcal{S}^{D_s-1}$ in feature space $\mathbb{R}^{D_s}$ (where $D_s$ need not equal $D_z$ in experiments). For survey $i$ belonging to patient $n_i \in \{1, \dots, N\}$, the sampling process is defined as:
\begin{align}
    z_n &\sim \mathcal{N}(0, \mathbb{I}_{D_s}), \quad v_n = \frac{(z_{n,1}, \dots, z_{n,D_s-1},\, \lvert z_{n,D_s} \rvert)^\intercal}{\lVert z_n \rVert} \in \mathcal{S}^{D_s-1}, \label{eq:synth_code} \\
    \lvert \mathcal{D}_n \rvert &\sim \operatorname{Uniform}\{k_{\min}, \dots, k_{\max}\}, \\
    x_{i,d} &= \operatorname{clip}_{[0,10]}\!\left(\operatorname{round}(5 + 2\varepsilon_{i,d})\right), \quad \varepsilon_{i,d} \sim \mathcal{N}(0,1), \quad d = 1, \dots, D_s, \\
    y_i &\sim \operatorname{Bernoulli}\!\left(\operatorname{sigmoid}\!\Big(\gamma \cdot \tfrac{1}{5}\sum_{d=1}^{D_s}(x_{i,d} - 5)\,v_{n_i,d}\Big)\right), \label{eq:synth_label}
\end{align}
where $x_i \in \{0, \dots, 10\}^{D_s}$ represents Likert survey responses. Taking the absolute value $\lvert z_{n,D_s} \rvert$ projects every direction onto the upper hemisphere, yielding ground-truth pairwise similarities $v_n^\intercal v_{n'}$. The parameter $\gamma$ controls decision boundary sharpness. Surveys are sampled per patient until at least three positive and three negative outcomes are observed.
For a visual illustration of this data-generating process, see \cref{fig:synthetic}.

\textbf{Cohort Construction.}
We generate $36$ synthetic cohorts by crossing six cohort sizes ($N \in \{10, 48, 86, 124, 162, 200\}$) with six ``boundary sharpness levels'' ($\gamma \in \{1, 4, 7, 10, 13, 16\}$), dictating label noise around the decision boundary. These two parameters govern how much ``signal'' is in the data. Each cohort uses $D_s = 2$ features and $\lvert\mathcal{D}_n\rvert \sim \operatorname{Uniform}\{12, \dots, 20\}$ surveys per patient, producing datasets with rough class label balance to ensure results are not confounded by the class imbalance present in real EMA data.

\textbf{Experimental Setup.}
Surveys are split, in generation order, into a $50/25/25$ train/validation/test partition, with features min--max scaled on training rows. Models are fit on each cohort across three random initializations, with the restart selected via the highest validation log-likelihood.

\textbf{Hyperparameters and Model Selection.}
The two evaluated variants differ \emph{only} in how $q(U)$ couples to latent codes $Z$, isolating the structural effect of variational coupling:
\begin{itemize}[leftmargin=*]
    \item \textbf{PG-LSGP (Structured $q(U \mid Z)$):} P\'olya-Gamma bound (\cref{apx:pg}) where $q(U \mid Z)$ attains the exact code-conditional optimum in closed form (\cref{eq:pg_moments}).
    \item \textbf{PG-LSGP (Factorized $q(U)$):} Factorized bound (\cref{apx:pg-factorized}) using a free, full-covariance Gaussian.
\end{itemize}
Both variants use $M = 100$ free inducing points, full-batch Adam optimizer updates ($2{,}000$ iterations, learning rate $0.01$), a product kernel $K_\theta = K^x_\theta(X, X') \odot K^z_\theta(Z, Z')$ (linear $K^x_\theta$, ARD-RBF $K^z_\theta$), and $D_z = 1$ (matching the manifold's intrinsic dimension). The code posterior $q(Z; \phi)$ uses a DeepSet encoder $g_\phi$ (2 hidden layers, width 64, 3 Monte Carlo samples per step), with $m_\theta(x) = b$ warm-started to the logit of the training prevalence. Encoder inputs are remapped from $[0, 1]$ to $[-1, 1]$ to prevent amortized posterior collapse.
Two necessary implementation differences exist: structured $q(U \mid Z)$ executes $T = 15$ inner CAVI sweeps per step, whereas factorized $q(U)$ requires gradient norm clipping at $1.0$ to stabilize its free $M \times M$ Cholesky factor.

\subsection{The EMA Dataset} \label{apx:data-ema}\label{apx:setup-u01}

\textbf{Patients and Recruitment.}
A total of 623 unique patients presenting with suicidal thoughts or recent suicidal behavior were recruited across \ifpreprintflag two Boston-area hospital sites\fi: 315 adults (ages 18+) from an emergency psychiatric service and 308 adolescents (ages 12--19) from an inpatient unit. Patients who re-enrolled are counted once, even if they were assigned more than one study ID. Exclusions included non-smartphone ownership, lack of English fluency, inability to consent, or severe cognitive impairment (e.g., psychosis, acute intoxication, dementia).

\textbf{Consent, Compensation, and Safety Protocol.}
Patients completed baseline assessments, installed the \ifpreprintflag LifeData smartphone app\fi, and received \$10 for baseline completion plus \$1 per submitted survey. Study protocols were approved by the institutional review board. A dedicated clinical risk team monitored real-time responses to intervene upon high reported suicidal intent.

\textbf{EMA Sampling and Survey Structure.}
Prompts were delivered 6 times daily for three months: a daily survey and a nighttime survey at fixed times as the first and last prompts of each day, with 4 randomized prompts in between. Patients could also submit voluntary self-initiated surveys (e.g., to report self-injury). Surveys evaluated 3 suicidal-thoughts items (urge, intent, ability to resist) and 17 affective states on a $0$--$10$ Likert scale.

\textbf{Suicide-Related Event (SRE) Outcomes.}
SREs could be identified from seven sources: three patient self-reports (a user-initiated survey, the daily survey, or the nighttime survey); the study's event log, maintained by the clinical team; emergency calls, in which clinicians learned of an event and logged it; reports relayed by study staff over the team's internal messaging channel; and, for patients who consented to record access, the electronic health record (EHR), consensus-coded by trained reviewers. An SRE is retained only if preceded by at least one EMA survey within 7 days.

\textbf{De-Identification.}
All data was de-identified before being shared with the research team: patients are indexed only by a study ID, with no names, dates of birth, or other direct identifiers retained. The demographic-annotated similarity figures we report (\cref{fig:graph,fig:heatmap} and their appendix variants) are drawn entirely from these de-identified records; each patient appears only as a node or a row keyed by study ID, colored by a broad demographic category shared by many patients, so no figure in this paper displays a name, a raw survey response, or any information beyond a patient's de-identified demographic category and their position in the learned similarity structure.

\textbf{Data Cleaning and Cohort Construction.}
When patients explicitly skipped suicidal-thoughts items, the instrument recorded a skip; these surveys were retained, and missing items were imputed using each patient's individual training-data mean to capture within-person levels without introducing between-person differences. Two kinds of surveys were dropped instead. First, a survey in which a suicidal-thoughts item was missing \emph{without} a recorded skip: we do not know why the value is absent, so imputing it would tell the model the question was answered. Second, a skipped survey from a patient who never answered that item anywhere in their training data, since there is then no personal mean to impute from.

This strategy retained 76,798 of 79,448 recorded surveys (a 3.3\% loss). In contrast, listwise deletion would have retained only 73,739 surveys, discarding 3,059 surveys where patients declined suicidal-thoughts questions---a response pattern plausibly tied to clinical risk and therefore not missing at random. Of 623 recruited patients, 503 contributed at least one EMA survey. Six were excluded for lacking valid responses required for person-mean imputation, resulting in a final analytical cohort of N = 497 patients and 76,798 surveys (81.7\% of patients experienced zero SREs).

\textbf{Features.}
Every model sees the same $D_x = 35$ columns:
\begin{itemize}[nosep, leftmargin=*]
    \item \emph{Affect items (17):} humiliated, isolated, negative, numb, desire to escape, energetic, fatigued, hopeless, trapped, worried, positive, self-hate, agitated, angry, burdensome, desire to approach people, places, and things, and desire to avoid people, places, and things.
    \item \emph{Suicidal-thoughts items (3):} urge to die by suicide, intent to die by suicide, and ability to resist dying by suicide.
    \item \emph{Survey metadata and context (8):} minutes from prompt to survey start, survey duration in minutes, days since the previous survey, daytime indicator, weekend indicator, holiday indicator, day-before-holiday indicator, and daylight duration.
    \item \emph{Other survey-level columns (4):} count of extreme affect responses, COVID-19 state-of-emergency indicator, number of SREs accumulated so far in the study, and adult-versus-adolescent indicator (by recruitment site).
    \item \emph{Skip indicators (3):} whether the urge, intent, and ability-to-resist items were skipped.
\end{itemize}

\textbf{Experimental Setup.}
Models are evaluated on these 35 columns across three training fractions, $h \in \{0.3, 0.5, 0.8\}$. Because temporal splitting is deterministic for a given $h$, we evaluate the three training fractions as distinct datasets representing varying amounts of history rather than reporting a mean and variance across them. Class prevalence remains stable across splits (train: $3.9\%$, $3.7\%$, $3.6\%$; test: $4.1\%$, $3.7\%$, $3.4\%$). Test windows contain 69, 57, and 31 SRE episodes across 51, 47, and 30 patients for $h = 0.3, 0.5, 0.8$, respectively.

Within each patient's trajectory, surveys are ordered chronologically. To prevent temporal leakage and avoid splitting multi-day clinical events across sets, split boundaries are snapped strictly outside any continuous 7-day SRE run. Each patient's timeline is then cut into \textsf{train} $\mid$ \textsf{val} $\mid$ \textsf{test} in time order, training taking a fraction $h$ and validation and test each taking half of what remains: $h_{\mathrm{val}} = h_{\mathrm{test}} = \frac{1 - h}{2}$.
Features are min--max standardized to $[0, 1]$ based on training set statistics. 

We fit non-deterministic models across 5 random initializations, selecting among them on validation log-likelihood for the LSGP and on validation per-patient AP for the baselines, for which the seed is tuned alongside the remaining hyperparameters.

The two criteria answer two different questions.
Restart selection asks which of several local optima of \emph{one} objective to keep: every LSGP variant maximizes the same evidence lower bound on the log marginal likelihood, so the held-out LL is the out-of-sample counterpart of that objective, and selecting on it mitigates local optima without introducing a criterion the objective never saw.
Model comparison asks which method forecasts within-patient risk, which per-patient AP measures (Result \ref{sec:results}.1), and this criterion is applied identically to every LSGP variant and every baseline.
The baselines have no comparable inner objective (their seed enters through data subsampling, not through a bound being maximized), which is why their seed is tuned as a hyperparameter instead.

\textbf{Hyperparameters and Model Selection.}
Each baseline's hyperparameters are selected via validation per-patient AP (see Result \ref{sec:results}.1). Shared LSGP settings include $M = 1024$ inducing points, mini-batch size $\vert{}\mathcal{B}\vert{} = 4096$, 10,000 Adam iterations ($\text{learning rate} = 0.005$), the product kernel $K_\theta = K^x_\theta \odot K^z_\theta$, linear $K^x_\theta$, a DeepSet encoder $g_\phi$ (2 layers, width 16), $D_z \in \{1, 2\}$, gradient norms clipped at 1.0, and a constant mean $m_\theta(x) = b$ warm-started to the logit of the training prevalence.
The LSGP ablation has no encoder to stabilize, so it is unclipped at $D_z = 1$ with a mean-field $q(Z; \phi)$, and it searches its $K^x_\theta$ under a zero-initialized linear mean $m_\theta(x) = w^\intercal x + b$. We give the (non-ablation) LSGPs the less expressive mean deliberately: a linear mean's per-feature slopes absorb the shared trend, which pre-explains the signal and, due to local optima, collapses the amortized codes onto the prior, whereas a scalar intercept is code-independent and cannot. The LSGP ablation has no codes to collapse, so it keeps the more expressive mean. One Monte Carlo code sample is drawn per step in every model; the two sparse-variational models additionally draw 10 samples of the latent function $f$, conditional on that single code, to estimate their reconstruction term---a term the P\'olya-Gamma bound does not sample, being closed form in $f$ once the codes and the auxiliary variables are fixed. The inner CAVI sweeps ($T = 15$) belong to the two models optimizing the P\'olya-Gamma bound---the deployed model and the LSGP ablation---the factorized bound and the sparse-variational model having no E-step.

We evaluate the LSGP against the two extremes it bridges, alongside targeted controls that isolate specific model capabilities; each model's search space follows.

\begin{itemize}[leftmargin=*]
    \item \textbf{PG-LSGP (Structured $q(U \mid Z)$):} Grid-searched over the $2 \times 2$ crossing of amplitude scaling with warping,
    \begin{align*}
        K^z_\theta \in \{\text{ARD},\ \sigma_\theta^2\!\cdot\!\text{ARD},\ \text{Warped},\ \sigma_\theta^2\!\cdot\!\text{Warped}\},
    \end{align*}
    each further crossed with whether the amortized encoder's own copy of the input features is re-centered from $[0,1]$ to $[-1,1]$ before it sees them (the kernel itself always sees the untouched $[0,1]$ features). This hyperparameter axis was motivated by preliminary experiments on synthetic data, which revealed that the amortized DeepSet encoder can otherwise collapse to the prior when processing strictly non-negative $[0, 1]$ inputs. 
    
    \item \textbf{PG-LSGP (Factorized $q(U)$):} Uses the factorized bound over the same $K^z_\theta$ and encoder-input grid as the structured model.
    \item \textbf{SV-LSGP:} Sparse-variational bound comparing mean-field vs.\ amortized $q(U \mid Z)$ (diagonal covariance, 10 latent-function samples per step), over the same $K^z_\theta$ and encoder-input grid.
    \item \textbf{LSGP Ablation ($K^z_\theta \equiv 1$):} Retains the full model's variational bound, inducing points, and optimizer, but pools across patients. Switching the similarity channel off makes the codes inert, so this variant drops the amortized code posterior for a mean-field one, needs no gradient clipping, and takes the default linear mean. It additionally searches $K^x_\theta \in \{\text{Linear},\ \text{ARD},\ \text{Arc-Cosine},\ \text{NN-Erf}\}$, held fixed by the personalized variants, which makes it the natural shared-GP reference rather than a single-factor ablation.
    
    \item \textbf{Logistic Regression Baselines (LR / VB-LR / EB-LR):} $L_2$-regularized logistic regression, with variational (VB-LR) and empirical-Bayes (EB-LR, via type-II maximum likelihood with a Laplace approximation) variants. Fitting each both nomothetically and idiographically captures the trade-off between masking individual dynamics and overfitting data-sparse patients (\cref{apx:related-work}). Evaluated in both modes without early stopping to ensure convergence rather than relying on default library iteration limits ($\texttt{tol} = 0$ for standard LR; $5{,}000$ iterations for VB-LR and EB-LR). Hyperparameter grids comprise:
\begin{itemize}[leftmargin=*]
    \item \emph{Standard LR:} Inverse $L_2$ regularization strength $C \in \{10^{-3},\, 3 \times 10^{-3},\, 10^{-2},\, 3 \times 10^{-2},\, 0.1,\, 0.3,\, 1,\, 3,\, 10,\, 30,\, 100\}$.
    \item \emph{Variational Bayes LR (VB-LR):} Gamma prior shape and rate $a, b \in \{1.0, 2.0\}$.
    \item \emph{Empirical Bayes LR (EB-LR):} Initial precision $\alpha \in \{10^{-12},\, 10^{-9},\, 10^{-6},\, 10^{-3},\, 10^{-2},\, 0.1,\, 0.5,\, 1,\, 2,\, 3,\, 5,\, 10,\, 30\}$.
\end{itemize}

\item \textbf{Tree Baselines (GBM / RF):} Gradient-boosted decision trees and random forests fitted at the cohort level. Because the input covariance $K^x_\theta$ is linear, these test for non-linear structure its input kernel omits. Evaluated in nomothetic mode using \texttt{HistGradientBoostingClassifier} ($300$ boosting rounds, no early stopping or withheld validation split; $\texttt{max\_leaf\_nodes} \in \{15, 31\}$, $\texttt{learning\_rate} \in \{0.05, 0.1\}$) and \texttt{RandomForestClassifier} ($300$ trees, inverse-frequency class weighting; $\texttt{max\_depth} \in \{\text{None}, 8, 16\}$, $\texttt{min\_samples\_leaf} \in \{1, 5\}$).

\item \textbf{Hierarchical LR Baselines:} We evaluate mixed-effects logistic regression as the standard bridge between nomothetic and idiographic modeling in clinical psychology, in two variants: a random-intercept model (\emph{intercept}) and a full random-intercept-and-slopes model (\emph{intercept+slopes}). The generative process, notation, and function-space (GP) form of both variants are given in \cref{apx:hlm-function-space}. We fix the fixed-effect prior scale to $\sigma_{\mathrm{fe}} = 2$ and grid-search the log-scale variance component prior $\sigma_{\mathrm{vc}} \in \{0.5, 1, 2\}$, which dictates the strength of hierarchical pooling across patients.
Models are fit via mean-field variational inference using the \texttt{statsmodels}~\citep{seabold2010statsmodels} Bayesian mixed GLM implementation.

\item \textbf{Base-Rate Baselines:} A featureless baseline that ignores survey responses entirely, separating \emph{identifying high-risk individuals} from \emph{forecasting temporal risk within a patient}: a prediction constant within a patient gives an uninformative temporal ranking yet can score well pooled, so a model matching its pooled metrics is merely stratifying patients rather than forecasting their dynamics. It predicts smoothed event rates derived from a conjugate Beta--Binomial prior $\operatorname{Beta}(\alpha\pi, \alpha(1-\pi))$, where $\pi$ is the global training prevalence and $\alpha$ acts as a prior pseudocount. The predictive rate $p_n$ is the posterior mean:
\begin{equation}
    p_n = \operatorname{clip}_{[\epsilon,\, 1-\epsilon]}\!\left( \frac{k_n + \alpha \pi}{\lvert \mathcal{D}_n \rvert + \alpha} \right), \qquad \epsilon = 10^{-3},
\end{equation}
for patient event count $k_n$ and survey count $\lvert \mathcal{D}_n \rvert$. In nomothetic mode (\emph{Nomothetic Base-Rate} in the figures), data is pooled so that every patient receives the uniform global prevalence $\pi$. In idiographic mode (\emph{Idiographic Base-Rate}), $p_n$ shrinks individual rates toward $\pi$, with $\alpha$ either grid-searched over $\{0,\, 0.5,\, 1,\, 2,\, 4,\, 8,\, 16,\, 32,\, 64,\, 128,\, 256\}$ or fit via empirical Bayes maximizing the Beta--Binomial marginal likelihood:
\begin{align}
\begin{split}
    \alpha^\star = \operatorname*{arg\,max}_{\alpha} \sum_{n=1}^{N} \Big[
    &\log \Gamma\big(k_n + \alpha\pi\big) + \log \Gamma\big(\lvert \mathcal{D}_n \rvert - k_n + \alpha(1-\pi)\big) - \log \Gamma\big(\lvert \mathcal{D}_n \rvert + \alpha\big) \\
    &- \log \Gamma\big(\alpha\pi\big) - \log \Gamma\big(\alpha(1-\pi)\big) + \log \Gamma\big(\alpha\big) \Big] ,
\end{split}
\end{align}
written with $\log\Gamma$ rather than the Beta function so that \emph{Beta} refers only to the prior distribution.
\end{itemize}

\section{ADDITIONAL RESULTS} \label{apx:results}

\textbf{Structured Bound Behavior and Factorized Baselines.}\label{apx:structured-collapse}
While \cref{thm:oversim} proves that the factorized bound carries an explicit optimization penalty $I_{\mathcal{A}}(\eta(Z))$ that vanishes only at the all-ones similarity (actively pulling inference toward personalization collapse), it makes no such prediction for the structured bound. The structured objective eliminates this artificial penalty, allowing the model to either personalize or collapse strictly as dictated by the data likelihood. 
On this cohort, structured variants can also reach personalization collapse, $\mathrm{MODS}(A^{\mathrm{comb}}) = 1$; for example, the selected structured SV variant collapses at $h = 0.8$ (\cref{fig:q8}). We interpret this behavior through three distinct mechanisms:
\begin{enumerate}[leftmargin=*,itemsep=2pt,topsep=2pt]
    \item \textbf{Data Homogeneity:} A collapsing structured variant does not contradict our theoretical framework. In a cohort where $81.7\%$ of patients experience no SREs, individual trajectories may lack sufficient heterogeneity to warrant personalization, rendering the nomothetic model likelihood-optimal. Crucially, \cref{thm:gap,thm:oversim} demonstrate that an all-ones similarity matrix is far more credible under the structured bound: because the structured objective contains no variational penalty favoring collapse, an all-ones matrix reflects a less confounded, likelihood-driven outcome rather than an artifact of mean-field optimization.
    \item \textbf{Optimization Artifacts in Baselines:} Factorized references achieve substantially lower per-patient LL than structured variants, indicating convergence to local minima away from the global optimum analyzed in \cref{thm:gap}. Consequently, lower $\mathrm{MODS}$ in factorized baselines may reflect optimization noise rather than genuine personalization missed by the structured bound. Furthermore, the low-prevalence LL dynamics (Result \ref{sec:results}.1) prevent using log-likelihood alone to rank these fits.
    \item \textbf{Likelihood as the Training Signal:} Both bounds are maximized as surrogates for the marginal likelihood, and Result \ref{sec:results}.1 shows that on this cohort the likelihood barely distinguishes a model that forecasts within a patient from one that matches the prevalence: a constant predictor is already near the best LL any model attains. The training objective therefore exerts little pull toward personalization for \emph{either} bound, and a small residual gain in likelihood from personalizing can be outweighed by the KL cost of the codes. Under such an objective the structured bound is free to collapse, and a metric that rewarded within-patient discrimination directly might not collapse where LL does.
\end{enumerate}
Definitively separating data-driven collapse from optimization local minima requires an exact optimum-versus-optimum comparison on real-world cohorts, which we leave to future work.

\begin{figure}[htb]
    \centering
    \includegraphics[width=\textwidth]{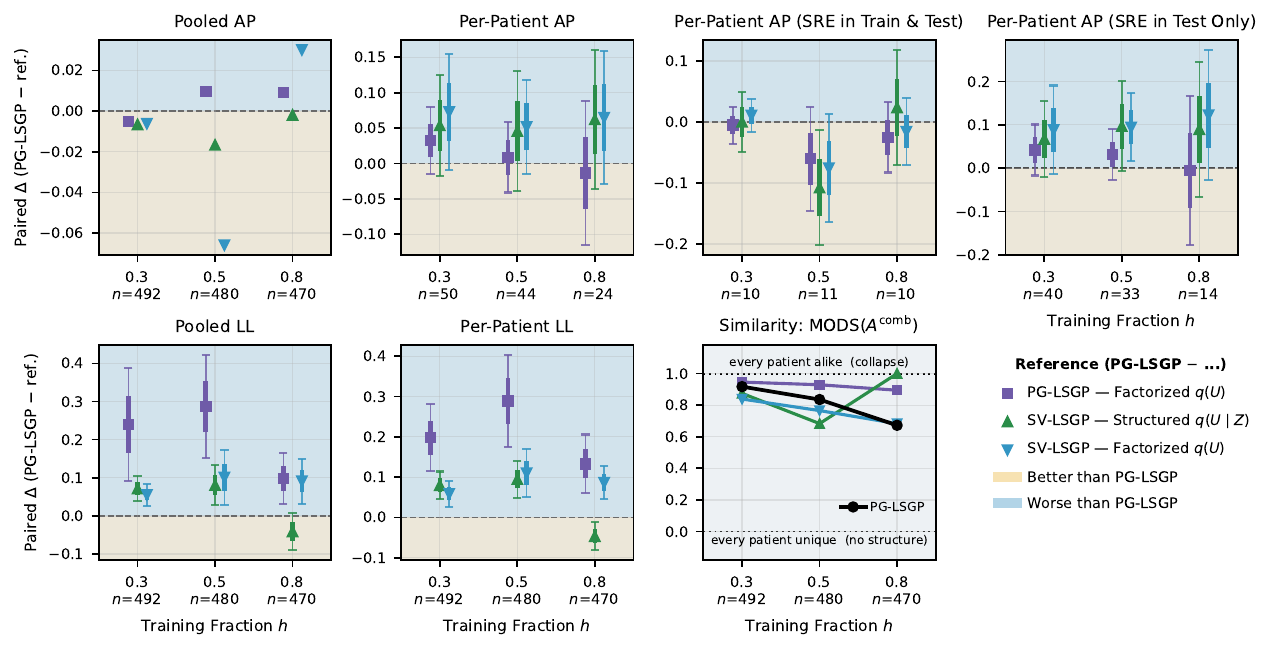}
    \caption{\textbf{Paired comparison of the structured PG-LSGP against the factorized and SV variants on the EMA data.} This plot is similar to \cref{fig:q7}: each marker is the mean paired per-patient difference (best-validated structured PG-LSGP minus the best-validated member of each reference family at each $h$), with thick and thin bars spanning $\pm 1$ and $\pm 2$ paired standard errors; the MODS panel shows each method's own similarity. Within the PG bound, adopting the structured over the factorized variational family improves per-patient LL by more than two paired standard errors at every $h$, reflecting superior predictive calibration, while per-patient AP differences remain largely within noise; the within-SV contrast is mixed. In the MODS panel, the factorized PG variant sits closer to collapse than the structured one at every $h$, yet the selected structured SV variant itself collapses at $h = 0.8$---neither class of models consistently exhibits personalization collapse (\cref{apx:structured-collapse}). Still, employing the structured family provides critical theoretical safeguards: it eliminates the optimization penalty against non-nomothetic fits, guaranteeing that any collapse is likelihood-driven rather than an artifact of the variational family.}
    \label{fig:q8}
\end{figure}

\begin{figure}[htb]
    \centering
    \includegraphics[width=\textwidth]{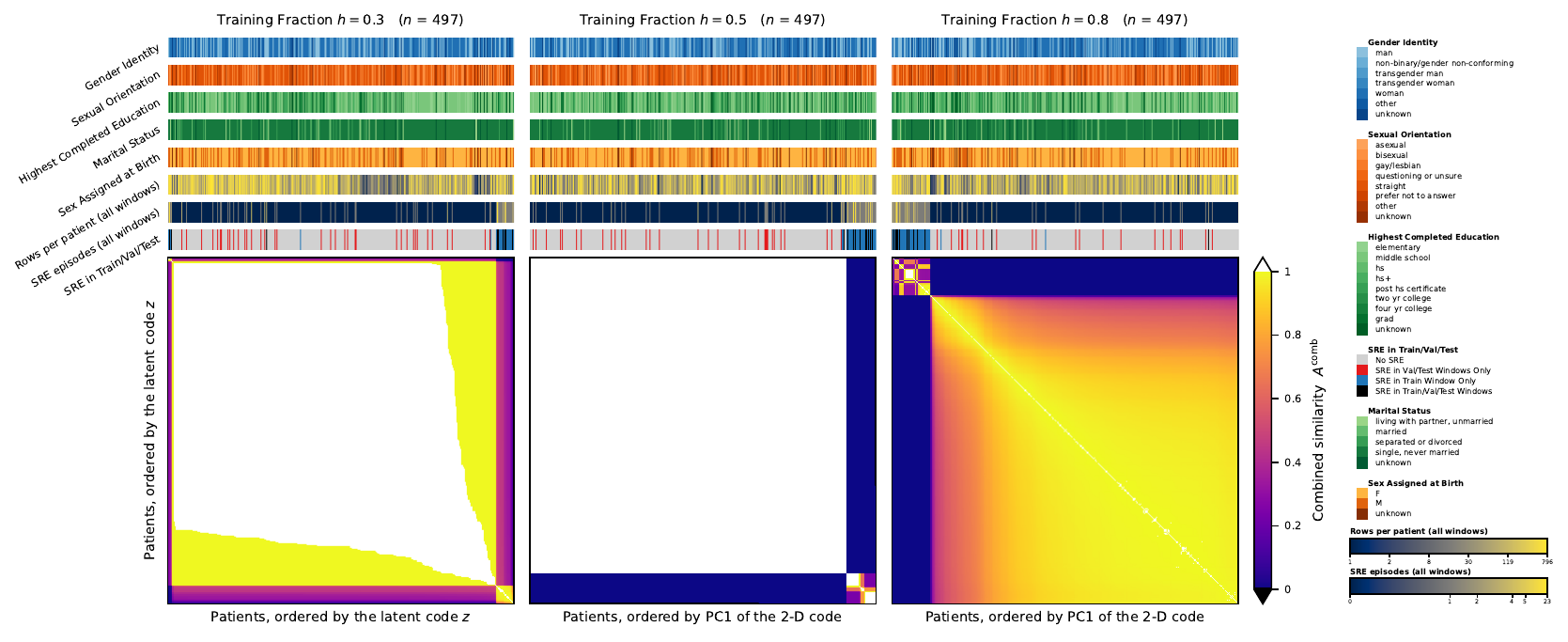}
    \caption{\textbf{Combined similarity ($A^{\mathrm{comb}}$) across the full cohort ($N = 497$).} Extending \cref{fig:heatmap} to every patient, across all training fractions $h$. Patients with training-window SREs exhibit distinct, well-defined similarity structures, while the remaining patients display less differentiated similarity, as the absence of positive training labels limits the model's ability to resolve fine-grained boundaries among lower-risk trajectories.}
    \label{fig:heatmap-extended}
\end{figure}

\begin{figure}[htb]
    \centering
    \includegraphics[width=0.99\textwidth]{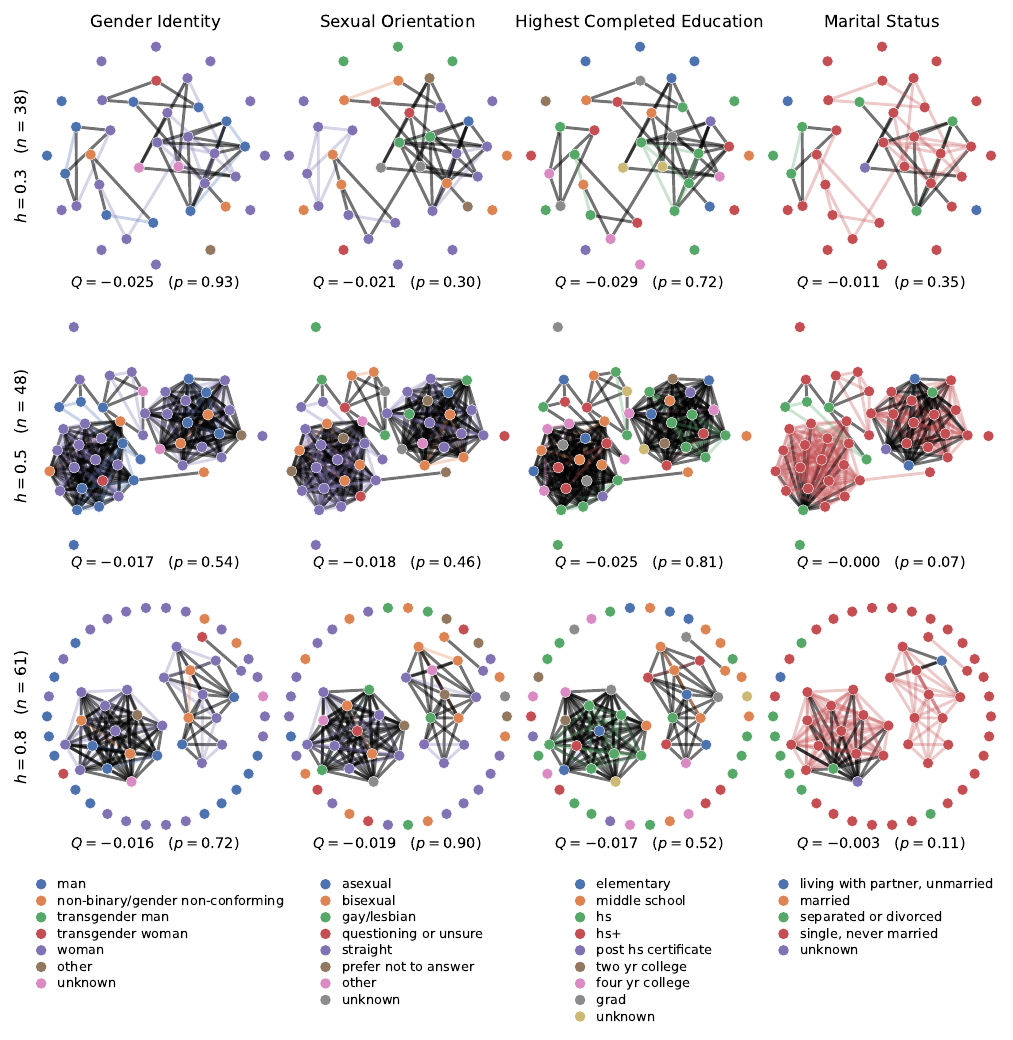}
    \caption{\textbf{Patient similarity graphs for the training-event cohort: demographic information does not explain patient similarity.} Nodes represent patients who experienced an SRE during their training window. Edge thickness reflects the learned combined similarity $A^{\mathrm{comb}}_{nm}$. For legibility, only the strongest edges are drawn: a single similarity threshold, shared across all panels so that an edge means the same similarity everywhere, is chosen so that each panel shows about six edges per patient. Graph modularity $Q$ and permutation $p$-values are computed on the complete similarity matrix $A^{\mathrm{comb}}$ before this pruning. Rows correspond to training fractions $h \in \{0.3, 0.5, 0.8\}$, each showing the variant with the highest validation per-patient AP at the restart selected on validation log-likelihood; columns correspond to demographic attributes (node colors).}
    \label{fig:graph}
\end{figure}

\end{document}